\documentclass{article}

\usepackage{amsmath}
\usepackage{amssymb}
\usepackage{amsfonts}
\usepackage{amsthm}
\usepackage{mathtools}
\usepackage{bm}
\usepackage{bbm}
\usepackage{dsfont}
\usepackage{cases}
\usepackage{nicefrac}
\usepackage{MnSymbol}
\usepackage{bbding}
\usepackage{pifont}

\usepackage[utf8]{inputenc} 
\usepackage[T1]{fontenc}    
\usepackage{textcomp}

\usepackage{graphicx}
\usepackage{wrapfig}
\usepackage{float}
\usepackage{caption}
\usepackage{subcaption}     
\usepackage{array}
\usepackage{multirow}
\usepackage{multicol}
\usepackage{booktabs}
\usepackage{makecell}
\usepackage{colortbl}

\usepackage{xr-hyper}
\usepackage{hyperref}       
\usepackage[capitalize,noabbrev]{cleveref}
\usepackage{url}

\usepackage[makeroom]{cancel}

\hypersetup{
  final,
  colorlinks,
  linkcolor=ourred,
  citecolor=ourblue,
  urlcolor=url,
}

\usepackage{algorithm}
\usepackage{algpseudocode}

\usepackage[page,toc,titletoc,title]{appendix}
\usepackage{microtype}
\usepackage{xargs}
\usepackage{xspace}
\usepackage[normalem]{ulem}
\usepackage{soul}
\usepackage{fancyvrb}

\usepackage{lipsum}
\usepackage{blindtext}

\usepackage{xcolor}
\definecolor{ourblue}{rgb}{0.368,0.507,0.71}
\definecolor{ourorange}{rgb}{0.881,0.611,0.142}
\definecolor{ourgreen}{rgb}{0.56,0.692,0.195}
\definecolor{ourgreen2}{rgb}{0.46,0.792,0.195}
\definecolor{ourred}{rgb}{0.923,0.386,0.209}
\definecolor{ourviolet}{rgb}{0.528,0.471,0.701}
\definecolor{ourbrown}{rgb}{0.772,0.432,0.102}
\definecolor{ourlightblue}{rgb}{0.364,0.619,0.782}
\definecolor{ourbrightblue}{RGB}{93,135,237}
\definecolor{ourdarkgreen}{rgb}{0.572,0.586,0.}

\definecolor{ourcyan2}{rgb}{0.125,0.722,0.804}
\definecolor{ourred2}{rgb}{0.863,0.184,0.047}
\definecolor{ouryellow2}{cmyk}{0,0.16,1.0,0.07}
\definecolor{ourviolet2}{cmyk}{0.55,0.56,0,0.47}
\definecolor{ourorange2}{cmyk}{0,0.46,0.89,0.11}
\definecolor{grayseq}{RGB}{120,120,120}

\definecolor{discretecolor}{RGB}{11,83,150}
\definecolor{gaussiancolor}{RGB}{230,145,56}
\definecolor{argmaxcolor}{RGB}{154,0,0}

\definecolor{customred}{RGB}{169, 45, 59}
\definecolor{url}{HTML}{d95225}

\definecolor{olivegreen}{RGB}{165,185,115}
\definecolor{olivegreendark}{RGB}{145,165,95}

\definecolor{maroon}{RGB}{140, 45, 45}

\usepackage[most]{tcolorbox} 
\usepackage{soul}

\makeatletter
\DeclareRobustCommand\onedot{\futurelet\@let@token\@onedot}
\def\@onedot{\ifx\@let@token.\else.\null\fi\xspace}

\makeatother

\newtcolorbox{corollarybox}{
  colback=ourbrightblue!12,
  colframe=ourbrightblue!80!black,
  left=3.5pt,
  right=3.5pt,
  top=3pt,
  bottom=2pt
}

\newtcolorbox{findingbox}{
  breakable,
  enhanced,
  colback=olivegreen!12,
  colframe=olivegreen!70!black,
  left=2pt,
  right=2pt,
  top=2pt,
  bottom=2pt,
  boxrule=1.5pt
}

\usepackage{amsmath,amsfonts,bm}

\def\eqref#1{equation~\ref{#1}}

\def\1{\bm{1}}

\DeclareMathAlphabet{\mathsfit}{\encodingdefault}{\sfdefault}{m}{sl}
\SetMathAlphabet{\mathsfit}{bold}{\encodingdefault}{\sfdefault}{bx}{n}

\def\x{{\mathbf x}}
\def\z{{\mathbf z}}

\def\h{{\mathbf h}}

\def\cat{\text{Cat}}
\def\x{{\mathbf x}}
\def\h{{\mathbf h}}

\def\xi{{\mathbf x}^{(i)}}

\def\m{{\mathbf m}}

\def\at{{\alpha_{t}}}

\newcommand{\onehotset}{\mathcal{V}}

\def\Nc{N_\mathrm{cont}}
\def\Nd{N_\mathrm{disc}}

\def\supl{\ell}
\def\xtl{\x_t^{\supl}}

\def\u{\mathbf{u}}
\def\v{\mathbf{v}}

\def\xts{\x_s^{\supl}}
\def\xsl{\x_s^{\supl}}

\def\xl{\x^{\supl}}

\def\xi{\x^{\supi}}

\def\pmaskbert{p_{\text{mask}}^{\tilde{\x}}}
\def\sigmaregbert{\sigma_\text{reg}^{\tilde{\x}}}
\def\pmaskz{p_{\text{mask}}^{\z}}
\def\pdropoutz{p_{\text{dropout}}^{\z}}

\usepackage{fancyvrb}
\usepackage{pythonhighlight}

\usepackage[preprint]{neurips_2026}

\usepackage[utf8]{inputenc} 
\usepackage[T1]{fontenc}    
\usepackage{hyperref}       
\usepackage{url}            
\usepackage{booktabs}       
\usepackage{amsfonts}       
\usepackage{nicefrac}       
\usepackage{microtype}      
\usepackage{xcolor}         

\usepackage{microtype}
\usepackage{graphicx}
\usepackage{booktabs} 
\usepackage{makecell}

\definecolor{bluish}{HTML}{2955BE}
\definecolor{nvgreen}{HTML}{89CC16}
\definecolor{nvblue}{HTML}{176CBD}

\title{DiLaDiff: Distilled Latent-Augmented Diffusion for Language Modeling}

\author{%
  Jean-Marie Lemercier\thanks{Corresponding author: \texttt{jlemercier@nvidia.com}} \\
  NVIDIA\\
  \And
  Tomas Geffner \\
  NVIDIA \\
  \And
  Morteza Mardani \\
  NVIDIA \\
  \And
  Karsten Kreis \\
  NVIDIA \\
  \And
  Arash Vahdat \\
  NVIDIA \\
  \And
  Ante Juki\'{c} \\
  NVIDIA
}

\begin{document}

\maketitle

\begin{abstract}
    Diffusion language models intrinsically fail to capture correlations between decoded tokens, which leads to a harsh trade-off between sampling quality and throughput.
    To solve this issue, we propose \textbf{DiLaDiff}, a variant of masked diffusion language models with three components: (\textit{1}) a continuous latent space with semantic capabilities, learned by an auto-encoder fine-tuned from an existing masked diffusion language model; (\textit{2}) a latent diffusion model learning the prior over the encoder distribution; (\textit{3}) a consistency model distilling the learned prior into a few-step latent generative model.
    We show that, even without distillation, our latent-guided diffusion model outperforms the masked diffusion baseline while significantly accelerating inference.
    Consistency distillation further lowers the computational overhead of continuous diffusion, such that the latent is generated in negligible time compared to discrete decoding\footnote{Project page: \texttt{https://jmlemercier.github.io/diladiff/}}.
\end{abstract}

\section{Introduction}

Diffusion-based generative models~\citep{sohl2015deep, song2019generative, ho2020denoising, song2021scorebasedgenerativemodelingstochastic} have emerged as a dominant framework for modeling continuous data such as natural images~\citep{dhariwal2021diffusion}, videos~\citep{ho2022video}, speech \citep{popov2021gradttsdiffusionprobabilisticmodel} or protein design~\citep{watson2023}. 
Continuous diffusion paradigms offer many benefits, such as principled samplers through differential equation solvers as well as distillation techniques \citep{song2023consistencymodels, geng2024consistencymodelseasy, geng2025meanflowsonestepgenerative}.
Building on this success, recent works have explored extending diffusion models to categorical data~\citep{austin2023structureddenoisingdiffusionmodels,campbell2022continuoustimeframeworkdiscrete,lou2024discretediffusionmodelingestimating,gat2024discreteflowmatching,sahoo2024simpleeffectivemaskeddiffusion}.
While this formulation is well-suited to discrete data, it forgoes the aforementioned advantages of continuous latent representations.
More crucially, discrete diffusion models cannot decode many tokens in parallel accurately, since the denoiser model parameterizing the posterior marginals does not capture token dependencies~\citep{israel2025acceleratingdiffusionllmsadaptive,liu2025discretecopuladiffusion,xie2025variationalautoencodingdiscretediffusion}.
Another line of works consider continuous diffusion on text representations, such as
one-hot encodings~\citep{lee2026flowmaplanguagemodels}, token-wise embeddings~\citep{li2022diffusionlm, gulrajani2023likelihoodbaseddiffusionlanguagemodels, chen2026langflowcontinuousdiffusionrivals} and contextual representations \citep{meshchaninov2026cosmoscompressedsmoothlatent}, thereby recovering principled sampling and distillation techniques.
Despite these advantages, continuous models have been lagging behind their discrete counterparts.
This has been mostly attributed to the intrinsic categorical nature of language, and also to the trade-off between expressivity and decodability of the continuous representation of text sequences~\citep{zhou2025coevolutionarycontinuousdiscretediffusion}.
On one end of the trade-off, token embeddings are easily decodable but do not extract contextual information.
On the other end, contextual encodings, e.g., the output states of deep layers in a representation model \citep{zhou2025coevolutionarycontinuousdiscretediffusion,meshchaninov2026cosmoscompressedsmoothlatent,shariatian2025latentdiscretediffusionmodels,kang2025ladirlatentdiffusionenhances}, provide an expressive latent but are notoriously harder to decode back to the original sequence. 
\citet{zhou2025coevolutionarycontinuousdiscretediffusion,morris2023textembeddingsrevealalmost} suggest using a diffusion process to simplify the task of decoding the contextual latent. 

Based on these insights, we propose a model combining a continuous latent prior over compressed contextual representations, and a token-space discrete decoder. 
The model captures global correlations via continuous diffusion over contextual representations while benefiting from robust decoding in the token space.
In addition, (\textit{1}) the decoder can be initialized from a robust pre-trained discrete diffusion model; (\textit{2}) controllability of continuous diffusion trajectories in the latent space can be exploited; (\textit{3}) ODE trajectories in the latent space can be distilled and continuous paradigms can be used to improve sampling.
Our contributions in this work are as follows:
\begin{itemize}
    \item We propose a recipe for training a text auto-encoder, where the decoder is initialized from an pre-trained discrete diffusion baseline. This results in a controllable and regularized latent space capturing semantic information.
    
    \item We learn the latent prior with a continuous diffusion model, yielding a hybrid continuous-discrete diffusion model suited for sampling several tokens in parallel without sacrificing text quality. 
    The resulting model \textbf{LaDiff} (\textit{\textbf{La}tent-augmented \textbf{Diff}usion Language Model}) largely outperforms the masked diffusion baseline in unconditional generation, while accelerating inference by a factor of $7\times$ at batch size 32.
    
    \item We leverage the properties of continuous diffusion to further distill the latent diffusion component into a few-step generative model, resulting in \textbf{DiLaDiff} (\textit{\textbf{Di}stilled \textbf{La}tent-augmented \textbf{Diff}usion Language Model)}. 
    This is the first successful attempt at distilling ODE trajectories for contextual text latents.
    DiLaDiff obtains in only 5 steps a performance close to its LaDiff teacher using 200 steps.
    Distillation lowers the overhead of continuous diffusion, such that the latent variable is sampled in negligible time compared to discrete decoding.
    
\end{itemize}

\begin{figure}[t]
    \centering
    \includegraphics[trim={60 300 190 280},clip,width=0.9\linewidth]{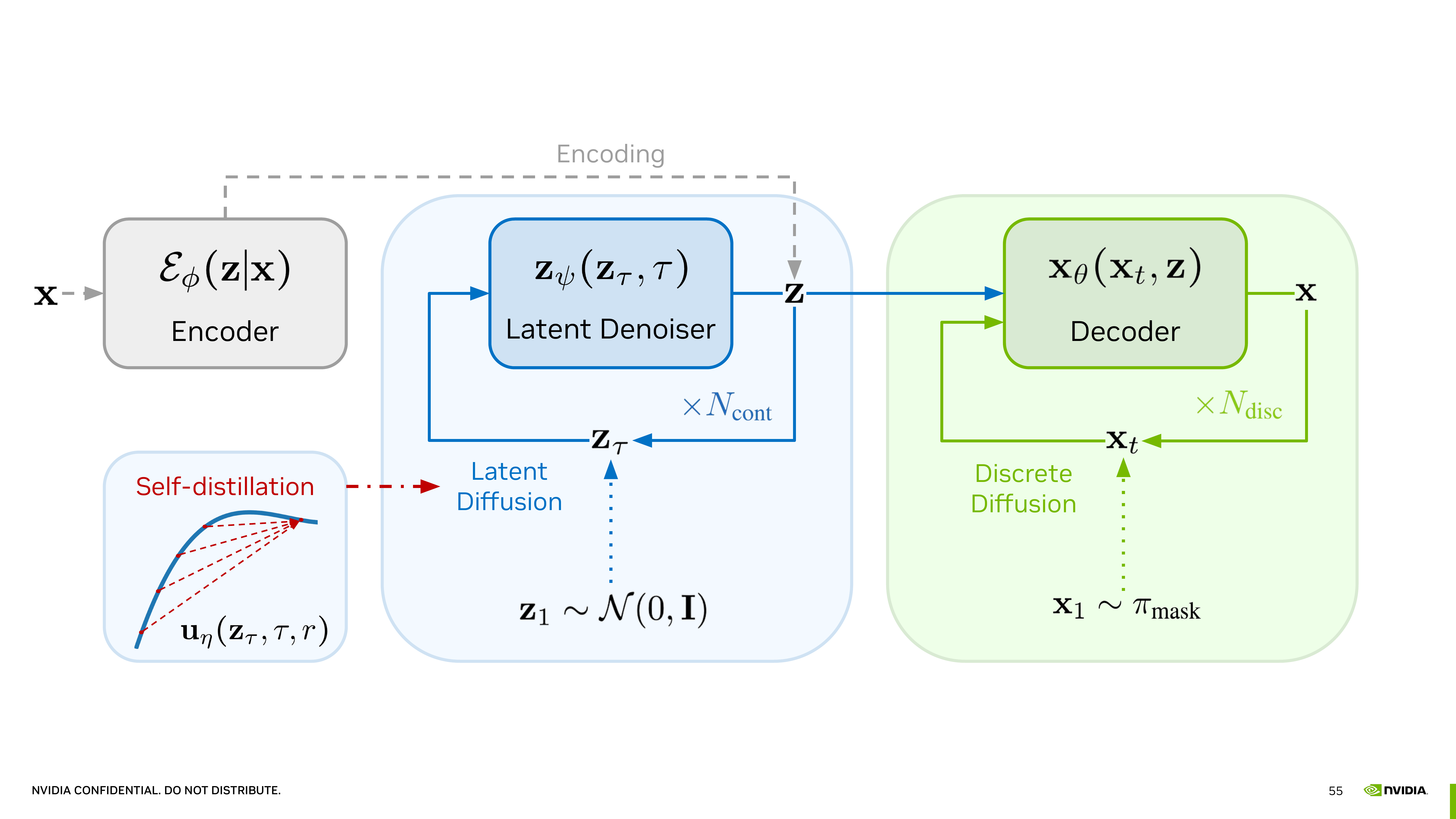}
    \caption{DiLaDiff: hybrid continuous-discrete diffusion with self-distilled latent. The latent space is crafted with encoder $\mathcal{E}_\phi$ and decoder $\x_\theta$ and learned a posteriori with a diffusion process with denoiser $\z_\psi$. The latent diffusion trajectories are further self-distilled with MeanFlow student $\u_\eta(\z_\tau, \tau, r)$.}
    \label{fig:diladiff}
\end{figure}

\section{Background}

We represent language tokens as one-hot encoded vectors in $ \onehotset \;\triangleq\; \left\{\mathbf{v}\in\{0,1\}^K:\sum_{i=1}^K v_i = 1\right\}$ where $K$ is the size of the vocabulary, and $v_i$ is the $i$-th element of the vector $\mathbf{v}$. 
We use $\x \in \onehotset^L$ to denote a sequence of $L$ tokens and $\xl$ is the corresponding token at the index $\ell$.




\subsection{Masked Diffusion Language Models}
\label{sec:mdlm}

Masked diffusion language models (MDLMs) are a variant of discrete diffusion models that has emerged as a simple yet powerful paradigm for language modeling.
These models learn by recovering tokens that were randomly masked in a text token sequence, using the unmasked tokens as context.
MDLMs are trained using an evidence lower-bound (ELBO) on the model likelihood:
\begin{equation}\label{eq:difflm}
\mathcal{L}_\mathrm{MDLM}^\theta =
\mathbb{E}_{\substack{t \sim \mathcal{U}[0,1] \\ \x_t \sim q_t( \x_t | \x )}} \frac{\dot{\alpha}_t}{1 - \alpha_t}
\sum_{\ell} \log \langle \xl_{\theta}(\x_t), \xl \rangle \geq - \log p_\theta(\x) \,,
\end{equation}
where $(\x_t)_{t \in [0,1]} \in \onehotset^L$ is a sequence with some of the tokens in $\x$ replaced by a special \texttt{[MASK]} token \citep{sahoo2024simpleeffectivemaskeddiffusion}, and $\x_\theta(\x_t)$ is the clean sequence estimate obtained from $\x_t$ by a denoising Transformer with parameters $\theta$.
The forward masking kernel $q_t( \mathbf{ \cdot } | \x )$ interpolates between the clean distribution at $t=0$ and a categorical prior $\mathbf{m}$ generating a fully masked sequence at $t=1$. This kernel is factorized independently for each token, with the marginal distribution for the $\ell$-th token:
\begin{equation}\label{eq:mdm-forward}
    q_t(\xtl|\xl) = \cat \left(\xtl; \at \xl + (1 - \at)\m\right) \, ,
\end{equation}
where $\mathbf{m}$ is a one-hot encoding of the \texttt{[MASK]} token.
The proportion of masked tokens at each time step $t$ is given by the commonly used linear noise schedule $\at = 1 - t$.
Samples are obtained through the reverse posterior $q_{s\mid t}(\x_s\mid \x_t)$, which can be marginalized as:
\begin{equation}\label{eq:posterior}
    q_{s\mid t}(\x_s\mid \x_t) = \int q_{s\mid t}(\x_s\mid \x_t, \x_0) q(\x_0 | \x_t) \mathrm{d}\x_0 \,.
\end{equation}
This posterior is intractable at inference time since we have no access to the true denoiser $q(\x_0 | \x_t)$.
Using a DDPM-style parameterization \citep{ho2020denoising}, MDLMs approximate the conditional denoiser $q(\x_0 | \x_t)$ as a product of token-wise marginals, using the denoising Transformer $\x_\theta$:
\begin{equation}\label{eq:ddpm}
    q^\theta(\x_0 | \x_t) := \prod_\ell \delta\left(\x_0^\ell = \x_\theta^\ell(\x_t) \right) \,,
\end{equation}
which, in turn, yields the following approximation for the posterior:
\begin{equation} \label{eq:mdm-post-factorization-approx}
    q_{s\mid t}^\theta(\x_s\mid \x_t) := \prod_\ell q_{s\mid t}\left(\xsl\mid \xtl, \x_0^\ell = \x_\theta^\ell(\x_t) \right) \,,
\end{equation}
with marginals \citep{sahoo2024simpleeffectivemaskeddiffusion}
\begin{equation}\label{eq:mdm-post-marginal-approx}
q_{s\mid t}(\xts|\xtl, \x_0^\ell = \x_\theta^\ell(\x_t)) :=
    \begin{cases}
        \cat (\xts; \xtl) & \xtl \neq \m, \\
        \cat \left(\xts; \frac{
        (1 - \alpha_s)\m + (\alpha_s - \alpha_t)\x_\theta^\ell(\x_t)}{1 - \alpha_t}\right) & \xtl=\m. 
    \end{cases}
\end{equation}

In theory, diffusion models carry the promise of parallel generation: unlike their auto-regressive counterparts, they are not theoretically limited to accepting one token per model forward.
However, in practice, diffusion models do not keep their promise because of the model mismatch between the factorized posterior $q_{s\mid t}^\theta$ in~(\ref{eq:mdm-post-factorization-approx}) and the true joint model $q_{s\mid t}$ in~(\ref{eq:posterior}).
Indeed, even if $\x_\theta$ perfectly recovers $\x_0$, the two posteriors are only equivalent if 
the denoiser marginals $q(\x_0^\ell | \x_t)$
are independent for all timesteps. 
This assumption is not true in practice, since tokens are heavily correlated in natural language, even when some clean (unmasked) context is available in $\x_t$.
The gap between the joint posterior and its factorized approximation is especially large toward the beginning of generation where the clean context in $\x_t$ is scarce.
This results in a trade-off between sampling quality and throughput \cite{nie2025largelanguagediffusionmodels}, as sampling more than one token per model call inevitably produces incoherent sentences.

\subsection{Continuous Latent Space for Text}
\label{sec:continuous}
Another paradigm for diffusion-based language modelling is to embed the discrete text tokens into a continuous space
, e.g., using auto-encoding, pre-trained contextual representations or one-hot embedding, and to perform diffusion in the latent space.
Latent samples $\z$ are diffused using a time-dependent perturbation kernel $\tilde{q}_t( \z_t | \z)$, and passed to a diffusion denoiser $\z_\psi$ trained with the following ELBO:
\begin{equation}\label{eq:gaussian_diffusion}
    \mathcal{L}_\mathrm{LDM}^\gamma = 
    \mathbb{E}_{\z \sim \mathcal{E} (\z | \x) }
    \mathbb{E}_{\substack{ t \sim \mathcal{U}[0, 1] \\ \z_t \sim \tilde{q}_t(\z_t | \z) }}
    || \z_\psi(\z_t, t) - \z ||_2^2 \, .
\end{equation}
At inference, a latent $\z$ is obtained by reverse diffusion,
then passed to the text decoder, e.g., learned decoder in the case of auto-encoding, or \texttt{argmax} operator in the case of one-hot embeddings.
This formulation enjoys many properties of continuous diffusion, in particular the ability to distill the resulting probability-flow ODE trajectory \citep{song2023consistencymodels,geng2025meanflowsonestepgenerative}.

\begin{figure}
    \centering
    \includegraphics[width=\linewidth]{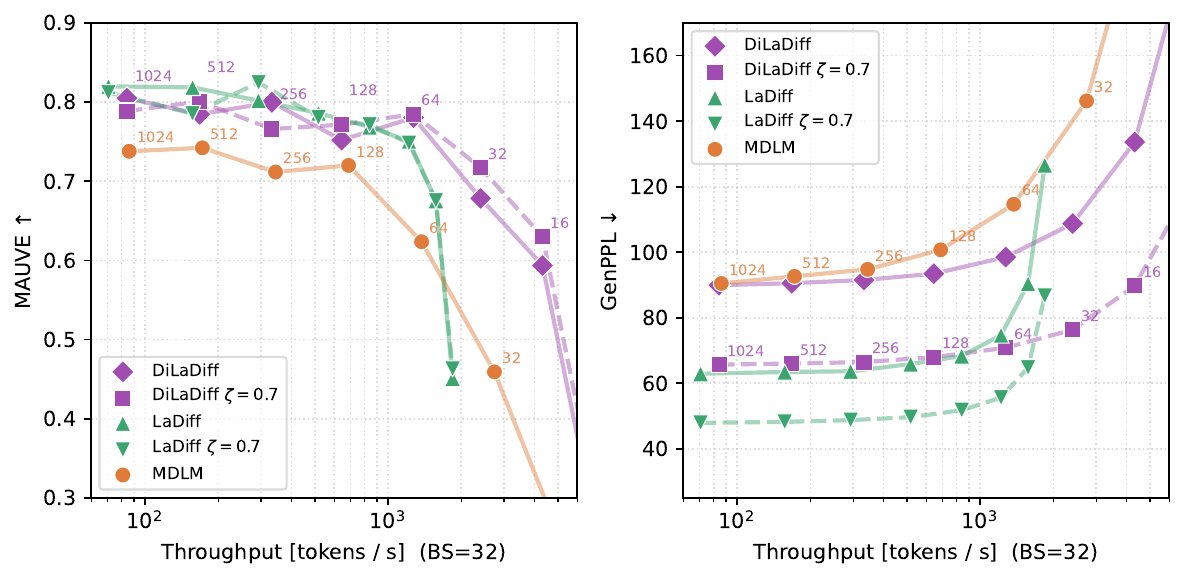}
    \caption{Speed-Quality Pareto frontier, for batch size $\text{BS}=32$. Labels next to points denote the number of discrete decoding steps $\Nd$. $\zeta$ is the temperature for scaling the discrete decoder's logits.}
    \label{fig:pareto32}
\end{figure}

\section{Latent-Augmented Diffusion Language Models}
\label{sec:vdlm}

\subsection{Auto-encoding text}

We propose to condition a masked diffusion language model by a latent learned via auto-encoding.
The training loss emerges from combining traditional auto-encoding with~(\ref{eq:difflm}):
\begin{equation}\label{eq:ae}
    \mathcal{L}_\mathrm{AE-MDLM}^{\phi, \theta} = 
    \mathbb{E}_{\substack{t \sim \mathcal{U}[0,1] \\ \x_t \sim q_t( \x_t | \x )}} 
    \frac{\dot{\alpha}_t}{1-\at}
    \mathbb{E}_{\z \sim \mathcal{E}_\phi (\z | \x) } 
\sum_{\ell} 
    \log \langle \x_\theta^\ell(\x_t, \z), \xl \rangle \,.
\end{equation}
The introduction of the latent channel $\z$ effectively turns the point-wise DDPM denoiser~(\ref{eq:ddpm}) into a denoiser conditioned on the latent $\z$:
\begin{equation}\label{eq:denoiser}
    q^\theta(\x_0|\x_t) :=  \int \delta\left(\x_0 = \x_\theta(\x_t, \z) \right) p(\z) \mathrm{d}\z = \int \Big( \prod_\ell \delta\left(\x_0^\ell = \x_\theta^\ell(\x_t, \z) \right) \Big) p(\z) \mathrm{d}\z \,,
\end{equation}
where $p(\z)$ is the latent prior. This yields in turn the following posterior:
\begin{equation}\label{eq:lmdm-rev-full}
    q_{s\mid t}^{\theta}(\x_s|\x_t) := \int \Big( \prod_\ell  q_{s\mid t}\left(\xsl|\xtl, \x_0^\ell = \x_\theta^\ell(\x_t, \z) \right) \big) p(\z) \mathrm{d}\z \,,
\end{equation}
Contrarily to (\ref{eq:mdm-post-factorization-approx}), the posterior (\ref{eq:lmdm-rev-full}) truly factorizes over independent tokens \textit{conditioned on the latent}, since the latent $\z$ captures the token correlations in $\x_0$.
In particular, the token correlations captured by the latent $\z$ are crucial for high masking ratios $t$, i.e., when clean context is scarce, since this is where the mismatch between the joint denoiser $q(\x_0|\x_t)$ and its marginals is the largest (see Figure~\ref{fig:reconstruction}).
The full proof is in Appendix~\ref{sec:proof}.

\paragraph{Latent space regularization} \label{sec:reg}

We base this work on \citet{meshchaninov2026cosmoscompressedsmoothlatent}, where pre-trained BERT hidden states are passed as input features to a Perceiver-inspired encoder \citep{jaegle2021perceivergeneralperceptioniterative} which further compresses the resulting text representation.
Such representations cannot be directly learned via diffusion because of the fundamental trade-off between representation capacity and smoothness, well documented in the auto-encoding literature \citep{kingma2014autoencodingvariationalbayes}.
Solely optimizing for reconstruction naturally pushes latents away from each other, resulting in a latent space too sparse for a generative model to learn.
We therefore employ the same robustness-enforcing heuristics as \citet{meshchaninov2026cosmoscompressedsmoothlatent}:  (\textit{1}) coordinate-wise normalization of BERT features $\tilde{\x}$ and latents $\z$; (\textit{2}) random masking of BERT features with probability $\pmaskbert$, and similarly for latents with probability $\pmaskz$; (\textit{3}) injection of Gaussian noise into BERT features, with standard deviation $\sigmaregbert$.
Finally, we further regularize training by employing a pre-trained MDLM decoder, and replacing the latent channel altogether with normally distributed noise, with chance $\pdropoutz$.
More details and pseudo-code for auto-encoder training can be found in Section~\ref{sec:exp_training} and Appendix~\ref{sec:details_ae}.

\subsection{Modelling text latents}

The clean sequence is not available at inference, and therefore we need to learn the latent prior $p(\z)$.
We propose to use the diffusion framework presented in Section~\ref{sec:continuous} for learning text latents produced by our auto-encoder, i.e., learn a parametric prior $p_\gamma(\z) \approx \mathbb{E}_{\x} \left[ p_\phi(\z|\x) \right]$.
As opposed to \citet{xie2025variationalautoencodingdiscretediffusion} where $\z$ is directly sampled from an isotropic normal distribution, our model can capture complex dependencies using a learned latent prior, and therefore provide better guidance to the decoder.
The proposed \textbf{LaDiff} model (\textbf{La}tent-augmented \textbf{Diff}usion Language Model) offers a natural 
interpolation scheme between pure continuous diffusion in the latent space and discrete mask-based diffusion in the token space.
Pseudo-code and training details of the proposed model are listed in Appendix~\ref{sec:details_ldm}.

At inference time, we first sample a latent sample $\z$ by solving the reverse probability-flow ODE using our latent denoiser $\z_\psi$ (trained with (\ref{eq:gaussian_diffusion})) and time discretization $\{\tau_{0 \leq m \leq \Nc} \}$, where $\Nc$ is the number of continuous diffusion steps. Updates have the following form:
\begin{equation}
    \forall m = \Nc, \dots, 1 : \quad \z_{\tau_{m-1}} = \text{ODESolver}(\z_{\tau_{m}}, \tau_{m-1}, \tau_m, \z_\psi) \,.
\end{equation}
We then perform ancestral sampling in the token space, starting from a fully masked sequence and progressively sampling the reverse diffusion posterior~(\ref{eq:lmdm-rev-full}) with time discretization $\{t_{0 \leq n \leq \Nd} \}$, where $\Nd$ is the number of discrete diffusion steps:
\begin{equation}
    \forall n = \Nd, \dots, 1 : \quad \x_{t_{n-1}} \sim q^\theta_{ t_{n-1} \mid t_{n} } \left( \x_{t_{n-1}} | \x_{t_{n}}, \x_0 = \x_\theta^\ell(\x_{t_n}, \z_{\tau_0}) \right) \,.
\end{equation}
A complete pseudo-code for LaDiff sampling procedure is available in Algorithm~\ref{alg:ladiff-sampling}.

\subsection{Distilling latent trajectories}
\label{sec:distillation}

A key advantage of using continuous latents is the ability to distill the corresponding ODE trajectories.
In this work, we use MeanFlow \citep{geng2025meanflowsonestepgenerative} as it naturally allows multi-step generation (as opposed to consistency models).
MeanFlow tasks our student \textbf{DiLaDiff} (\textbf{Di}stilled \textbf{La}tent-augmented \textbf{Di}ffusion Language Model) with learning the \textit{average} velocity $\u(\z_t, t, r)$, i.e., the average displacement between two points $t$ and $r$ on the ODE path:
\begin{equation}
    \mathbf{u}(\z_t, t, r) := \frac{1}{t-r} \int_r^t \mathrm{d}\z_\tau \,.
\end{equation}
In the many-step regime $r \rightarrow t$, the average velocity $\u(\z_t, t, r)$ converges to the \textit{instantaneous} velocity $\mathbf{v}(\z_t, t) := \frac{\mathrm{d}\z_\tau}{\mathrm{d}\tau}|_{\tau=t}$, so they can be used interchangeably to sample from the model.
However, in the few-step regime the two velocities are not aligned anymore, and the average velocity $\mathbf{u}$ must be used in place of the instantaneous velocity $\mathbf{v}$ in order to stay on the ODE path.
MeanFlows can be trained from scratch, but we distill from an existing LaDiff teacher $\psi$, which yields the following objective:
\begin{equation}
    \mathcal{L}_\text{MeanFlow}(\eta) = \mathbb{E} || 
    \u_\eta(\z_t, t, r) - \text{stopgrad} ( \u_\text{tgt} )
    ||_2^2 \,,
\end{equation}
where the target $\u_\text{tgt}$ is obtained from the instantaneous velocity and the average velocity derivatives:
\begin{equation}
    \u_\text{tgt} = \v_\psi(\z_t, t) - (t-r) \left( \v_\psi(\z_t, t) \partial_\mathbf{z} \u_\eta + \partial_t \u_\eta \right) \,.
\end{equation}
Considering first-order Euler ODE sampling for simplicity, samples are obtained by replacing the teacher's instantaneous velocity $\mathbf{v}_\gamma$ with the student's average velocity $\mathbf{u}_\eta$:
\begin{equation}
    \forall m = \Nc \, \dots, 1 : \quad \z_{\tau_{m-1}} = \z_{\tau_m} + (\tau_{m-1}-\tau_m) \, \mathbf{u}_\eta(\z_{\tau_m}, \tau_{m}, \tau_{m-1}) \,.
\end{equation}
Complete training pseudo-code for DiLaDiff is available in Algorithm~\ref{alg:diladiff-sampling}.

\section{Experiments}\label{sec:exp}

We perform language modelling experiments on OpenWebText \citep{owt}, where the last 100K documents of the dataset are used as a held-out validation set. Sentences are packed to length $L=1024$.
We use the \texttt{bert-base-uncased} tokenizer, with a vocabulary size of $K = 30522$, thus matching the setup in \citet{meshchaninov2026cosmoscompressedsmoothlatent}.

\paragraph{Training}\label{sec:exp_training}

Discrete diffusion baselines, MDLM~\citep{sahoo2024simpleeffectivemaskeddiffusion} and DUO~\citep{sahoo2025diffusionduality}, are re-implemented and trained for 1M steps with global batch size of 512 as in \citet{sahoo2025diffusionduality}. We use a cosine-decay learning schedule with 1000 linear warm-up steps with a maximal learning rate of 3e-4.
Auto-encoders are then fine-tuned from the MDLM checkpoint for additional 200k steps with a learning rate of 5e-5.
The architecture is a modified version of \citet{meshchaninov2026cosmoscompressedsmoothlatent}. In the encoder, the latent variable $\z$ is learned via cross-attention to the hidden state $\h$. In the decoder, we insert cross-attention layers, wrapped in zero-initialized pointwise convolutional layers, to enable the decoder's hidden state to capture the latent information.
For both MDLM and the auto-encoder, we replace the ELBO weighting $\frac{\dot{\alpha}_t}{1-\at}$ by $-1$, as \citep{sahoo2026scaling} empirically demonstrate it lowers the objective variance and therefore insures better convergence.
Additional implementation details for auto-encoders can be found in Appendix~\ref{sec:details_ae}.
LaDiff is trained for 150k steps with a learning rate of 2e-4, with the auto-encoder's weights frozen. 
We use the parametric \mbox{tanh-logSNR} variance-preserving noise schedule~\citep{hoogeboom2023simplediffusionendtoenddiffusion} with a grid search for the warping parameter $d$ (see Appendix~\ref{sec:additional_ae}).
Additional implementation details for schedules can be found in Appendix~\ref{sec:details_ldm}.
DiLaDiff is self-distilled for 25k steps using frozen LaDiff teachers, with a learning rate of 5e-5 and 25\% pure flow-matching loss ($t=r$). 
Because of the modified self-conditioning mechanism in DiLaDiff, one call of DiLaDiff requires two latent denoiser NFEs. Therefore $\Nc=5$ in DiLaDiff has a computational complexity similar as $\Nc=10$ for LaDiff (the exact computational overhead is given in Table~\ref{tab:throughputs}).
More details regarding self-conditioning, training procedure and hyperparameters can be found in Appendix~\ref{sec:details_distill}.

\paragraph{Evaluation}\label{sec:exp_eval}

We solve the probability-flow ODE with first-order Euler. For few-step sampling with DiLaDiff, we use $\gamma$-sampling \citep{kim2024consistencytrajectorymodelslearning} with $\gamma=0.8$ (see Appendix~\ref{sec:additional_distill}).
We perform ancestral sampling for discrete diffusion, with a temperature of $\zeta=1$ and nucleus filtering with $p=0.9$.
We report generative perplexity (GenPPL) using GPT2-Large \citep{Radford2019LanguageMA} as the scoring model.
The GenPPL metric may be misleading in degenerate cases, e.g., when redundant text is generated~\citep{zheng2025maskeddiffusionmodelssecretly}.
Therefore, we also report token-level entropy to complement GenPPL and to detect degenerate cases.
As a semantic coherence and all-in-one metric, we also report MAUVE \citep{pillutla2021mauvemeasuringgapneural}, using the traditional \texttt{gpt2-large} contextual embeddings.

\begin{minipage}{0.48\linewidth}
    \centering
    \includegraphics[width=\linewidth]{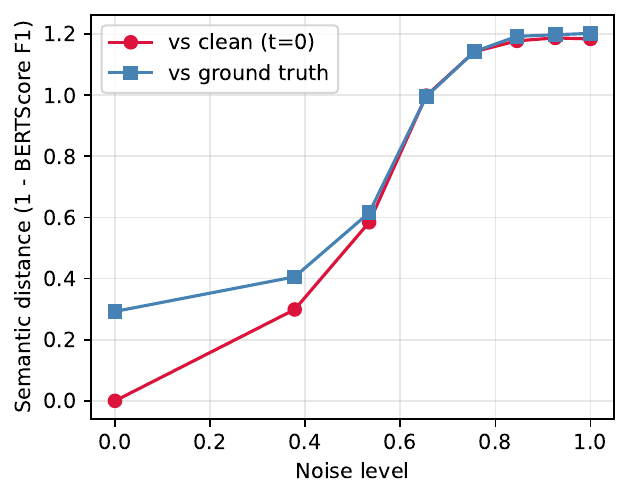}
    \captionof{figure}{Semantic distance between decoded sentence of noisy versions of the same latent.}
    \label{fig:semantic}
\end{minipage}
\hfill
\begin{minipage}{0.48\linewidth}
    \centering
    \includegraphics[width=\linewidth]{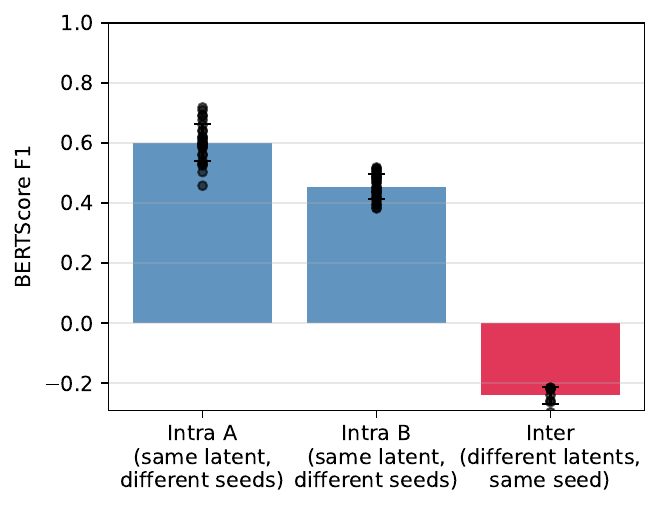}
    \captionof{figure}{Semantic similarity decoded sentences from same or different latents.}
    \label{fig:intra_inter}
\end{minipage}

\subsection{Learning text representations}\label{sec:exp_ae}

We evaluate the reconstruction performance of our auto-encoders on the held-out validation set of OpenWebText.
The performance is evaluated in terms of the token recovery rate and the corresponding lower bound on perplexity (see Appendix~\ref{sec:ppl_ae}).
As shown in Table~\ref{tab:ldm_ae}, our auto-encoders have a much lower perplexity than the baseline MDLM, which is expected as they have access to the clean sequence $\x$.
Furthermore, the perplexity and token recovery rate improve as the size of the latent space increases and reach a saturation point at around $2\times$ compression, corresponding to the latent size of $L/2 = 512$. 
The oracle sampling quality, evaluated by passing true encoder latents to the decoder, improves with latent space size.
Similarly, using little or no regularization provides the best perplexity and oracle performance (see complete ablation in Section~\ref{sec:ablations}).
However, this \textit{does not simply translate to better generative performance} when learning the encoder prior with diffusion. Indeed, increasing the latent channel capacity without adapting the regularization strength results in an ill-conditioned latent space which is hard to model with diffusion.
Using the right regularization strength, we can reach optimal sampling quality for $2\times$ compression and therefore use that setting for the remainder of the experiments.

\begin{table*}[b]
    \centering
    \begin{tabular}{l c cc cccc}
    \toprule
         & Compression & PPL ($\downarrow$) & \makecell{\%Recovery ($\uparrow$)} & \multicolumn{2}{c}{GenPPL} ($\downarrow$) & \multicolumn{2}{c}{MAUVE} ($\uparrow$) \\
         & & & & Oracle & Generated & Oracle & Generated \\
        \cmidrule(lr){5-6}\cmidrule(lr){7-8} 
         \textit{Data} & -- & -- & -- & -- & \textit{14.8} & - & \textit{1.00} \\
         \midrule
         MDLM & -- & $\leq$ 30.0 & 45.2 & -- & 90.5 & - & 0.74 \\
         \midrule
         LaDiff$^\dagger$ & $32\times$  & $\leq$ 6.83 & 53.3 & 61.2 & 120.9 & 0.86 & 0.77 \\
         LaDiff$^\dagger$ & $16\times$ & $\leq$ 6.83 & 53.3 & 48.7 & 112.5 & 0.89 & 0.80 \\
         LaDiff$^\dagger$ & $8\times$ & $\leq$ 3.28 & 61.0 & 45.9 & 80.2 & 0.89 & 0.77 \\
         LaDiff$^\dagger$ & $4\times$ & $\leq$ 2.58 & 68.3 & 38.9 & 93.9 & 0.89 & 0.74 \\
         LaDiff$^\dagger$ & $2\times$ & $\leq$ 2.25 & \textbf{75.2} & \textbf{31.4} & 85.1 & \textbf{0.90} & 0.73 \\
         LaDiff$^\dagger$ & $1\times$ & \textbf{$\leq$ 2.15} & 74.5 & 31.6 & 140.2 & \textbf{0.90} & 0.14 \\
         \midrule
         LaDiff & $2\times$ & $\leq$ 3.25 & 69.4 &  41.6 & \textbf{62.9} & 0.87 & \textbf{0.82} \\
         LaDiff & $1\times$ & $\leq$ 2.87 & 69.7 & 40.1 & 68.8 & \textbf{0.90} & 0.77 \\  
         \bottomrule
    \end{tabular}
    \caption{Validation recovery and test performance on OpenWebText. Models marked with $\dagger$ have auto-encoders trained with the MildAug regularization strategy (see Section~\ref{sec:ablations}). All models use $\Nc=200$ and $\Nd=1024$.}
    \label{tab:ldm_ae}
\end{table*}

\paragraph{Latent space analysis}

The auto-encoder latent is supposed to capture token correlations, i.e., some high-level representation of the sentence. We show here that this representation actually encodes some semantic information, and therefore that our encoder preserves some of the properties of the BERT feature during compression.
We take a sample from the encoder's posterior distribution, corrupt it with Gaussian noise with various noise levels, then decode it with $\Nd=8$ steps, keeping the decoder's stochasticity fixed. 
As reported on Figure~\ref{fig:semantic}, we observe monotonically increasing semantic distance, i.e., decreasing BERTScore-F1~\citep{zhang2020bertscoreevaluatingtextgeneration} between the decoded sentences (also, we qualitatively observed in this experiment that some words were being substituted for their synonyms). 
We also perform an analysis of the semantic diversity of our decoder when conditioned on a given latent. We sample 10 sentences from latent A with different seeds, and measure the pair-wise BERTScore-F1 within this pool of sentences. We then compare it to the pair-wise similarity between the "A" pool and a pool of sentences decoded from another latent B, using the same seed for decoding.
We observe in Figure~\ref{fig:intra_inter} that the sentences are much more closely related within a pool than between pools, but still exhibit some diversity, as shown by the deviation bars. 
These experiments suggest that the latent space carries semantic information and that the decoder has its own stochasticity, both of which properties are desirable for controllable and diverse generation.

\subsection{Hybrid continuous-discrete diffusion}\label{sec:exp_ldm}

\begin{minipage}{0.48\textwidth}
We place the performance of our hybrid models LaDiff and DiLaDiff on the speed-quality Pareto frontier
in Figure~\ref{fig:pareto32}, using a batch size $\text{BS}=32$ for generation and throughput measurement.
This enables taking into account the computational overhead incurred by the latent diffusion, for a fair comparison between pure discrete diffusion and our hybrid scheme.
In Table~\ref{tab:throughputs}, we compute the wall-time computational overhead of LaDiff's continuous diffusion component on a NVIDIA RTX 6000 Pro Blackwell GPU.
\end{minipage}\hfill
\begin{minipage}{0.5\textwidth}
\centering
\scalebox{0.9}{
\begin{tabular}{c ccc}
    \toprule
     &       \multicolumn{2}{c}{LaDiff} & DiLaDiff \\
     \cmidrule(lr){2-3}  \cmidrule(lr){4-4}
    \makecell{Batch\\size} & \makecell{$\Nc=200$ \\ $\Nd=1024$} & \makecell{$\Nc=200$ \\ $\Nd=64$} & \makecell{$\Nc=5$ \\ $\Nd=64$} \\
    \cmidrule(lr){1-4}
    1 & 0.127 & 1.894 & 0.091 \\
    8 & 0.075 & 1.175 & 0.053 \\
    32 & 0.070 & 1.116 & 0.050 \\
   \bottomrule
\end{tabular}
}
\captionof{table}{Computational overhead of latent continuous diffusion reported as a fraction of the wall-time duration of the corresponding discrete diffusion with $\Nd$ steps.}
\label{tab:throughputs}
\end{minipage}

In the performance-optimal regime using $\Nc=200$ for continuous diffusion (i.e., no distillation) and $\Nd=1024$ for discrete decoding (i.e., equal to the sequence length $L$) a 7\% overhead is allocated to continuous diffusion in the latent space (cf. Table~\ref{tab:throughputs}, column 2). This setting results in a 28.5 absolute (30\% relative) GenPPL improvement and 0.08 absolute (10\% relative) MAUVE improvement compared to the masked diffusion baseline (MDLM) with $\Nd=1024$ steps (cf. Figure~\ref{fig:pareto32}).
When using fewer decoding steps, e.g., $\Nd=64$, the proposed LaDiff profits heavily from latent guidance, and consequently outperforms the optimal setting baseline MDLM with $\Nd=1024$ by 22.2 absolute (24.5\% relative) GenPPL and 0.035 absolute (4\% relative) MAUVE, while boasting a $7\times$ wall-time acceleration factor. 
However, in this setting, the continuous diffusion using $\Nc=200$ represents a 111\% computational overhead compared to the discrete decoding budget using $\Nd=64$. 

Although latent diffusion scales favorably with batch size over discrete decoding (due to latent space compression and the absence of logits projection / categorical sampling, see Table~\ref{tab:throughputs}), this highlights the importance of distillation.
For a fixed number of $\Nd=64$ discrete decoding steps, the proposed distilled model DiLaDiff obtains 0.79 MAUVE with only $\Nc=5$ latent diffusion steps (resulting in $2\Nc=10$ latent denoiser NFEs because of self-conditioning, see Appendix~\ref{sec:details_distill}), reducing the gap with its teacher LaDiff (MAUVE=0.81) with $\Nc=200$.
The resulting overhead of latent diffusion drops down to 5\%, i.e., the latent can be computed in negligible time, while improving GenPPL by 17.3 absolute (15\% relative) and MAUVE by 0.17 absolute (27\% relative) compared to MDLM with the same amount of discrete decoding steps $\Nd=64$.

In Table~\ref{tab:openwebtext-many}, we compare our methods to official baselines MDLM \citep{sahoo2024simpleeffectivemaskeddiffusion} and DUO (with greedy-tailed sampler)~\citep{sahoo2025diffusionduality}, reporting 
results from the corresponding papers.
Furthermore, in Table~\ref{tab:openwebtext-few} we also report the performance of few-step generative performance for distilled models MDLM+SDTT and DUO+DCD \citep{sahoo2025diffusionduality}.
It is important to note that we \textit{do not distill the discrete decoder}, while MDLM+SDTT and DUO+DCD exclusively focus on that aspect and therefore obtain good few-decoding-steps performance. 
We show here that even without explicit discrete distillation, DiLaDiff is able to rival these state-of-the-art few-step methods by using latent-augmented discrete diffusion.
Applying discrete distillation techniques on top of our distilled latent denoiser is a natural extension of the proposed model, and we leave it for future work.

\begin{minipage}[t]{0.47\textwidth}
    \centering
    \begin{tabular}{@{} l cc @{}}
        \toprule
        & GenPPL ($\downarrow$) & Entropy ($\uparrow$) \\
        \midrule
        \textit{Data} & \textit{14.8} & \textit{5.44} \\
        \midrule
        MDLM & 104.5 & 5.63 \\
        DUO & 71.7 & 5.22 \\
        MDLM (ours) & 90.5 & 5.50 \\
        SDDM (ours) & 75.3 & 5.31 \\
        \midrule
        LaDiff & 62.9 & 5.40 \\
        LaDiff $\zeta{=}0.7$ & 47.8 & 5.26 \\
        \bottomrule
    \end{tabular}
    \captionof{table}{Many-step generative performance on OpenWebText. $\Nc=200, \Nd=1024$.}
    \label{tab:openwebtext-many}
\end{minipage}%
\hfill
\begin{minipage}[t]{0.47\textwidth}
    \centering
    \begin{tabular}{@{} l cc @{}}
        \toprule
        & GenPPL ($\downarrow$) & Entropy ($\uparrow$) \\
        \midrule
        \textit{Data} & \textit{14.8} & \textit{5.44} \\
        \midrule
        MDLM+SDTT & 62.3 & 5.49 \\
        DUO+DCD & 46.3 & 5.38 \\
        \midrule
        DiLaDiff & 108.7 & 5.53 \\
        DiLaDiff $\zeta{=}0.7$ & 76.3 & 5.39 \\
        \bottomrule
    \end{tabular}
    \captionof{table}{Few-step generative performance on OpenWebText. DiLaDiff using $\Nc=5, \Nd=32$.}
    \label{tab:openwebtext-few}
\end{minipage}

\paragraph{Temperature sampling}

Using various discrete decoding strategies reveals stark contrasts between the baseline MDLM and the proposed LaDiff.
As shown in Figure~\ref{fig:temperature}, we observe the influence of temperature logits scaling and nucleus sampling in the discrete diffusion decoder on the quality-diversity trade-off.
When sampling from LaDiff, we can lower the temperature and still preserve sample entropy, which demonstrates that the latent captures diversity to a large extent (although the decoder hast its share of stochasticity, as shown in Section~\ref{sec:exp_ae}).
In comparison, MDLM's sample entropy rapidly dwindles when decreasing the temperature, as the diversity is entirely determined by the discrete diffusion process, and nucleus sampling (with a probability threshold of 0.9) is required to obtain reasonable sample quality at temperatures approaching $\zeta=1$. 

\begin{figure}
    \centering
    \includegraphics[width=\linewidth]{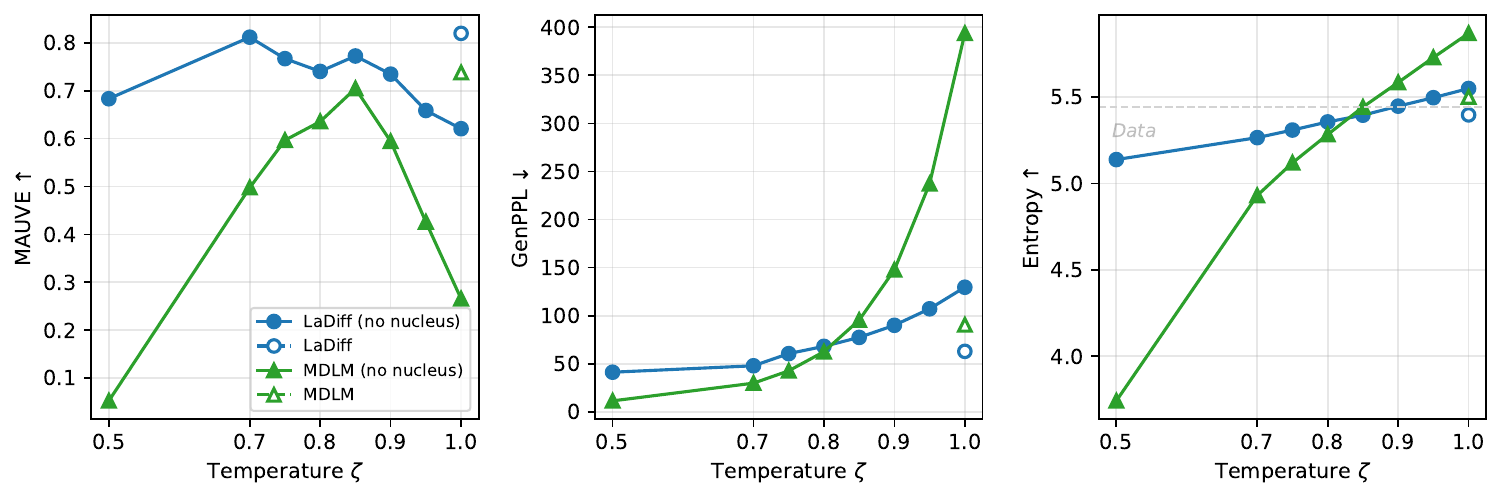}
    \caption{Temperature scaling and nucleus sampling. $\Nc=200$.}
    \label{fig:temperature}
\end{figure}

\subsection{Ablation studies}
\label{sec:ablations}

\paragraph{Auto-encoder training}

We ablate our design choices for the auto-encoding mechanisms described in Section~\ref{sec:reg}.
The influence of the regularization strategies is assessed by varying the parameters of the embedding and latent data augmentations and the results are presented in Table~\ref{tab:ablation_ae}.
Our best performing set of augmentation parameters is using
$\{ \sigmaregbert=0.5 \, , \pmaskbert=0.7, \pmaskz=0.7 \}$, as used in the proposed \mbox{LaDiff} model.
We consider two setups with a reduced amount of augmentation.
Firstly, we consider \mbox{MildAug} using $\{ \sigmaregbert=0.4 \, , \pmaskbert=0.5 \, , \pmaskz=0.5 \}$.
Secondly, we consider \mbox{SoftAug} using $\{ \sigmaregbert=0.3 \,, \pmaskbert=0.3 \, , \pmaskz=0.4 \}$, as used in \citet{meshchaninov2026cosmoscompressedsmoothlatent}).
Note that BERT features $\tilde{\x}$ examples in a batch are \textit{either} masked (half of the examples in a batch) or noised (the other half), not both.
The strongest regularization strategy of the three yields the best results: MildAug provides reasonable generative performance, but SoftAug regularization cannot craft a latent space suitable for diffusion, yielding poor MAUVE score.
Furthermore, we try reducing the chance of dropping the latent channel for decoding altogether, from $75\%$ to $50\%$ (cf. Dropout-50\% in Table~\ref{tab:ablation_ae}).
This has a similar effect as using the SoftAug regularization above: token recovery is high but generative performance is much worse than base LaDiff.
We also evaluated the impact of auto-encoder architecture by increasing the encoder depth from two to six layers (cf. BigEnc in Table~\ref{tab:ablation_ae}) and increasing the number of cross-attention layers in the decoder from one to three (cf. BigDec in Table~\ref{tab:ablation_ae}).
Both these modifications do not provide significant deviations, from which we conclude than even a minimal parameterization overhead is enough to learn a qualitative latent and condition the decoder on it.
For each auto-encoder, a latent denoiser is trained with an optimal, grid-searched diffusion schedule, which consistently yielded an optimal schedule parameter of $d=10$.

\begin{table}
    \centering
    \scalebox{0.95}{
  \begin{tabular}{lcccccc}
    \toprule
    & & & \multicolumn{2}{c}{Auto-encoding} & \multicolumn{2}{c}{Generation} \\
     \cmidrule(lr){4-5}\cmidrule(lr){6-7} 
    & Ablation setup & Compression & PPL ($\downarrow$) & \makecell{\%Recovery ($\uparrow$)} & GenPPL ($\downarrow$) & MAUVE ($\uparrow$) \\
    \midrule
    LaDiff & -- & 2x & $\leq$ 3.25 & 69.4 & \textbf{62.9} & \textbf{0.82} \\
    \midrule
    LaDiff & MildAug & 2x & $\leq$ 2.25 & 75.2 & 100.1 & 0.71 \\
    LaDiff & SoftAug & 2x & $\leq$ 1.92 & 81.1 & 98.5 & 0.24 \\
    LaDiff & Dropout-50\% & 2x & \textbf{$\leq$ 1.21} & \textbf{92.2} & 101.5 & 0.21 \\
    LaDiff & BigEnc & 2x & $\leq$ 2.88 & 76.7 & 65.1 & 0.78 \\
    LaDiff & BigDec & 2x & $\leq$ 3.01 & 75.6 & 64.5 & 0.80 \\
    \bottomrule
  \end{tabular}
  }
  \caption{Auto-encoder design ablations. $\Nc=200, \Nd=1024$.}
  \label{tab:ablation_ae}
\end{table}

\section{Conclusion}

We present (Di)LaDiff, a family of hybrid continuous-discrete diffusion language models.
LaDiff learns a semantically-consistent latent distribution of text, and uses it to guide a token-space diffusion language model.
By capturing token correlations at the sentence level, LaDiff substantially reduces the number of calls to the discrete decoder without sacrificing quality, outperforming its masked diffusion baseline while accelerating generation up to 7 times.
We further push acceleration by self-distilling the latent diffusion trajectories of LaDiff, yielding the few-step generation DiLaDiff model. 
DiLaDiff obtains results close to the LaDiff teacher using only 5\% of its latent denoiser calls, and comes within the reach of state-of-the-art models.

\paragraph{Limitations and future work}
Although DiLaDiff reduces the gap to its teacher with only a few steps, it does not match the topline teacher performance, suggesting more work is yet to be conducted on distillation.
Furthermore, 
we do not distill the discrete decoder, which is the main reason why true few-step performance (i.e., low latent diffusion \textit{and} low discrete decoding steps) is not state-of-the-art. This shall be explored in future work.
Finally, the presented strategies for regularizing auto-encoder training (see Section~\ref{sec:reg}), as well as the diffusion schedule design, are heuristic. We present a more principled attempt in Appendix \ref{sec:additional_ae}, but this requires more advanced work.

\section*{Impact statement}

This paper aims to improve language modeling. Potential societal consequences include disinformation or deep fakes, but this work does not specifically address nor raise any such issue.

\bibliography{bibliography}
\bibliographystyle{icml2025}


\appendix


\newpage

\newtheorem{theorem}{Theorem}[section]
\newtheorem{lemma}[theorem]{Lemma}

\section{Proofs} \label{sec:proof}

We prove here our main result in (\ref{eq:lmdm-rev-full}): 
\begin{theorem}[LaDiff reverse posterior]
    Given a latent vector $\z$ capturing correlations in $\x_0$, and the conditional denoiser   
\begin{equation*}
    q^\theta(\x_0|\x_t,\z) := \prod_\ell \delta\left(\x_0^\ell = \x_\theta^\ell(\x_t, \z) \right)\,,
\end{equation*}
the reverse posterior distribution of LaDiff is, for $s<t$:
\begin{equation*}
    q_{s\mid t}^{\theta}(\x_s|\x_t) := \int \Big( \prod_\ell  q_{s\mid t}\left(\xsl|\xtl, \x_0^\ell = \x_\theta^\ell(\x_t, \z) \right) \Big) p(\z) \mathrm{d}\z\,.
\end{equation*}
\end{theorem}

\newcommand{\eqnote}[1]{\textcolor{gray}{\footnotesize #1}}

\begin{proof}

We first prove the following lemma following \citep{uziel2026crocodilcontinuousrobustconditioned}:
\begin{lemma}[Conditional independence of reverse posterior marginals] \label{lemma:conditional_independence}

Given a latent vector $\z$ capturing correlations in $\x_0$ through auto-encoding, such that the posterior distribution equals the product of its marginals:
    \begin{equation}\label{eq:posterior_independence}
        q^\theta(\x_0\mid\z) = \prod_\ell q^\theta(\x_0^\ell\mid\z) \,,
    \end{equation}
Then the reverse posterior conditioned on the latent $\z$ factorizes into its marginals:
\begin{equation}
    q^\theta_{s\mid t}(\x_s\mid\x_t,\z) =
    \prod_\ell q^\theta_{s\mid t}(\x_s^\ell\mid\x_t^\ell,\z) \,.
\end{equation}
\end{lemma}

We derive the shape of the conditional reverse posterior distibution $q^\theta(\x_s|\x_t,\z)$ using Bayes:
\begin{align}
    q^\theta(\x_s\mid\x_t,\z)
    &= \frac{q(\x_t\mid\x_s,\z) q^\theta(\x_s\mid\z)}{q^\theta(\x_t\mid\z)} \\
    &= \frac{q(\x_t\mid\x_s) q^\theta(\x_s\mid\z)}{q^\theta(\x_t\mid\z)} \,, \label{eq:bayes}
\end{align}
where the second equality results from $\z$ being an encoded view of $\x_0$, and from the Markov property $q(\x_t\mid\x_s,\x_0) = q(\x_t\mid\x_s)$.

We then derive the conditional distribution 
\begin{align}
    q^\theta_t(\x_t\mid\z) 
    &= \int q\left(\x_t\mid\x_0,\z) q^\theta(\x_0\mid\z\right) \mathrm{d}\x_0 \\
    &= \int q\left(\x_t\mid\x_0, \mathcal{E}_\phi(\x_0)) q^\theta(\x_0\mid\z\right) \mathrm{d}\x_0 \\
    &= \int q\left(\x_t\mid\x_0) q^\theta(\x_0\mid\z\right) \mathrm{d}\x_0
\end{align}
where the last equality comes from $\z$ being an encoded view of $\x_0$.
We then use the key assumption that $q^\theta(\x_0\mid\z)$ factorizes into its marginals $\prod_\ell q^q(\x_0^\ell\mid\z)$, which relies on the fact $\z$ captures the token correlations in $\x_0$ through auto-encoding. Using the forward kernel factorization, we obtain:
\begin{align}
    q^\theta(\x_s\mid\z) 
    &= \int q(\x_s\mid\x_0) q^\theta(\x_0\mid\z) \mathrm{d}\x_0 \\
    &= \int \Big( \prod_\ell q(\x_s^\ell\mid\x_0) q^\theta(\x_0^\ell\mid\z) \Big) \mathrm{d}\x_0 \\
    &= \prod_\ell \int \Big( q(\x_s^\ell\mid\x_0) q^\theta(\x_0^\ell\mid\z) \mathrm{d}\x_0 \Big) \\ 
    &= \prod_\ell q^\theta(\x_s^\ell\mid\z) \,,
\end{align}
with the fore-to-last equality resulting from Fubini's theorem.
Plugging these conditional distributions in (\ref{eq:bayes}) and using once again the forward kernel factorization, we obtain:
\begin{align}
    q^\theta_{s\mid t}(\x_s\mid\x_t,\z)
    &= \prod_\ell \frac{ q(\x_t^\ell\mid\x_s^\ell) q^\theta(\x_s^\ell\mid\z)}{q^\theta(\x_t^\ell\mid\z)} \\
    &= \prod_\ell q^\theta_{s\mid t}(\x_s^\ell\mid\x_t^\ell,\z) \,,
\end{align}
which proves Lemma~\ref{lemma:conditional_independence}.
In practice, that means that several tokens can be sampled in parallel and still provide consistent text since the denoiser is conditioned on an informative latent.

Parameterizing the conditional reverse marginal with our denoiser (\ref{eq:denoiser}), we obtain
\begin{equation}
    q^\theta_{s\mid t}(\x_s\mid\x_t,\z) 
    = \prod_\ell q_{s\mid t}\big(\x_s^\ell\mid\x_t^\ell,\x_0^\ell=\x_\theta^\ell(\x_t, \z) \big) \,.
\end{equation}
We emphasize that omitting the conditioning on $\z$ in this last equality would yield the MDLM reverse posterior in (\ref{eq:mdm-post-factorization-approx}) with erroneous token independence assumption.
The proof then follows by marginalization over $\z$.

\end{proof}

\newpage

\section{Additional related work}

\paragraph{Discrete diffusion}

Adaptation of continuous diffusion toward direct modelling of categorical samples was first proposed in \citet{austin2023structureddenoisingdiffusionmodels, hoogeboom2021argmaxflowsmultinomialdiffusion}, and the continuous-time Markov chain theory was developed in \citet{campbell2022continuoustimeframeworkdiscrete}. Extensions of the framework developed rapidly, including concrete score matching \citep{lou2024discretediffusionmodelingestimating} and discrete flow matching \citep{gat2024discreteflowmatching}. An essential step for bridging the gap to AR models was reached with mask-based diffusion practical enhancements \citep{sahoo2024simpleeffectivemaskeddiffusion, zheng2025maskeddiffusionmodelssecretly, shi2025simplifiedgeneralizedmaskeddiffusion, ou2026absorbingdiscretediffusionsecretly}, and with uniform-state diffusion \citep{sahoo2025diffusionduality, zhu2026simpledenoisingdiffusionlanguage}. Large mask-based diffusion language models LLaDa \citep{nie2025largelanguagediffusionmodels} and Dream \citep{ye2025dream} have become staples for larger-scale open research.

\paragraph{Continuous diffusion}

Efforts for embedding text diffusion in a continuous space operate on representations include probability distributions~\citep{cheng2025categoricalflowmatchingstatistical,jo2025continuousdiffusionmodellanguage}, one-hot encodings~\citep{lee2026flowmaplanguagemodels}
, token embeddings~\citep{li2022diffusionlm, dieleman2022continuousdiffusioncategoricaldata, gulrajani2023likelihoodbaseddiffusionlanguagemodels, chen2026langflowcontinuousdiffusionrivals} and contextual representations \citep{meshchaninov2026cosmoscompressedsmoothlatent,zhou2025coevolutionarycontinuousdiscretediffusion}.
Closely related to our work is \citet{chen2026langflowcontinuousdiffusionrivals} which defines self-distilled flow maps in a continuous token-wise embedding space. The main differences are: (\textit{1}) we consider a hybrid continuous-discrete model while they design a purely continuous model, therefore concentrating all the model complexity in the continuous modality; (2) we finetune rich contextual embeddings provided by BERT in a similar fashion as \citet{meshchaninov2026cosmoscompressedsmoothlatent}, while they learn shallow token-wise continous embeddings.
Also concurrent to our work is \citet{lee2026flowmaplanguagemodels}, especially because they propose a self-distillation technique using flow maps. However, they operate on one-hot encodings and consider continuous-only diffusion.

\paragraph{Inference-time adaptations}
Training-free adaptation of discrete language models consists in defining a modified joint mixture distribution combining an auxiliary model (often auto-regressive) and the original diffusion marginals. 
Examples of these methods are discrete copula diffusion \citep{liu2025discretecopuladiffusion}, energy-based diffusion \citep{xu_energy-based_2025} and adaptive parallel decoding \citep{israel2025acceleratingdiffusionllmsadaptive}.

\paragraph{Hybrid diffusion}

Hybrid models can be coarsely divided into (\textit{1}) cross-modality diffusion, where a parallel or joint diffusion process over discrete and continuous modalities of text is conducted, and (\textit{2}) cascaded-modality diffusion, where a continuous latent is first generated then used as guide for discrete diffusion.
Cross-modality diffusion studies include
\citet{pynadath2025candihybriddiscretecontinuousdiffusion, shariatian2025latentdiscretediffusionmodels, zhou2025coevolutionarycontinuousdiscretediffusion, zheng2025continuouslyaugmenteddiscretediffusion}.
The most related technique from this group is \citet{zhou2025coevolutionarycontinuousdiscretediffusion}, as they use a LLM pre-trained contextual representation as their continuous modality. 
However, they do not provide any generation results, focusing on perplexity alone, and they do not fully generate the latent before guiding the discrete diffusion.

Our work is an example of cascaded-modality diffusion, and other such approaches are presented in \citet{xie2025variationalautoencodingdiscretediffusion, shariatian2025latentdiscretediffusionmodels, uziel2026crocodilcontinuousrobustconditioned}.
\citet{xie2025variationalautoencodingdiscretediffusion} introduce a latent channel learned via variational auto-encoding. Although this effectively captures correlations via auto-encoding, the resulting prior is not learned a-posteriori by an additional generative model and instead a standard Gaussian sample is used as latent during inference. Furthermore, in accordance with \citet{meshchaninov2026cosmoscompressedsmoothlatent}, we found that using traditional KL-based variational regularization for the latent space can often result in mode collapse, and is less efficient than heuristic regularization techniques.
\citet{shariatian2025latentdiscretediffusionmodels} learn the latent prior with a continuous model, however they use a frozen pre-trained encoder in their language modelling experiments, train the decoder from scratch, and provide no results on complex language benchmarks such as e.g. OpenWebText.
Closely related to the present work is CroCoDiL \citep{uziel2026crocodilcontinuousrobustconditioned} which we regard as concurrent to ours, although evolving in a different experimental framework (they focus on Python coding with 8B-scale models).

\clearpage 

\section{Training details}
\label{sec:details}

\subsection{Auto-encoding}\label{sec:details_ae}

Following \citet{meshchaninov2026cosmoscompressedsmoothlatent}, the BERT encodings and latents are standardized coordinate-wise using statistics aggregated along training.
This serves the triple-purpose of (\textit{1}) regularizing the latent space by avoiding extreme dispersion of latent vectors; (\textit{2}) enabling variance-preserving diffusion schedules for learning the latent and \textit{(3)} facilitating comparison between different designs and analyses.

We use a modified version of the architecture in \citet{meshchaninov2026cosmoscompressedsmoothlatent}. In each of its layers, the original encoder architecture models the latent variable $\z$ through cross-attention between the encoder's hidden state $\h$ and a concatenation of the hidden state $\h$ and the latent $\z$ itself, initialized as a free set of learned vectors.
The original version of the decoder's architecture is symmetrical with respect to the encoder. The decoder's hidden state is initialized as the BERT embedding of a fully masked sequence, and it attends to a concatenation of itself and the latent channel through the following cross-attention mechanism:

\begin{equation} \label{eq:hcrossatt}
    \h \leftarrow \h + \mathrm{CrossAttention}(\h; \left[ \h, \z \right]) \,.
\end{equation} 

We leave the encoder unchanged but bring the following modifications to the decoder design:
\begin{itemize}
    \item We initialize the decoder's hidden state with the partially masked sequence $\x_t$, in order to expose the decoder to various levels of masking, as required in our training objective (\ref{eq:ae}).

    \item We use the pre-trained decoder's embedding table to encode the masked sequence $\x_t$, instead of BERT embeddings. We freeze this layer during auto-encoder training to improve stability.

    \item We separate the self-attention and cross-attention mechanisms in the decoder. We noticed in preliminary experiments that the original decoder architecture is not able to properly leverage the clean context in $\x_t$ and instead pays excessive attention to the latent $\z$, degrading the marginal utility of our pre-trained MDLM decoder. 
    Therefore, we simply insert a few cross-attention layers between the original MDLM decoder's self-attention layers. 
    As the decoder's hidden state is already modelled through self-attention throughout the decoder, we modify these cross-attention layers so as to only extract information from the latent channel.
    Finally, in order to avoid perturbing the decoder's hidden state early in training, we wrap these new cross-attention layers in zero-initialized pointwise convolution layers, similar as \citet{zhang2023addingconditionalcontroltexttoimage}:
    
    \begin{equation} \label{eq:crossatt}
        \h \leftarrow \h + \mathrm{ZeroConv}\left( \mathrm{CrossAttention}\left( \mathrm{ZeroConv}\left( \h\right) ; \z\right)\right)  \,.
    \end{equation} 
    
\end{itemize}

\begin{algorithm}[t]
\caption{Auto-encoder training}
\label{alg:ae-mdlm-train}
\begin{algorithmic}[1]
\Require Frozen BERT, encoder $\mathcal{E}_\phi$,
         decoder $\mathbf{x}_\theta$, masking kernel $q_t(\cdot\mid\x)$,
         regularization hyperparameters $\pmaskbert, \sigmaregbert, \pmaskz, \pdropoutz$,
         feature statistics $(\boldsymbol{\mu}_{\tilde{\x}}, \sigma_{\tilde{\x}})$,

 \State \textcolor{red}{\textbf{// --- Encoder path ---}}
 
\State $\tilde{\x} \gets \text{BERT}(\x)$ 
    \Comment{BERT's last layer representation}
\State $\tilde{\x} \gets (\z - \boldsymbol{\mu}_{\tilde{\x}}) / \sigma_{\tilde{\x}}$
       \Comment{normalize feature}
\If{$\mathrm{rand}() < \tfrac12$}
    \State $\forall (\ell, k)\, : \, \tilde{x}^\ell_k \sim \delta\bigl( 1 - \text{Bernoulli}(\pmaskbert) \bigr) \tilde{x}^\ell_k$
    \Comment{feature masking}
\Else
    \State $\tilde{\x} \sim \mathcal{N}\!\left(\sqrt{1- (\sigmaregbert)^2 } \,\tilde{\x}\,;\; (\sigmaregbert)^2\,\mathbf{I}\right)$
           \Comment{feature noise}
\EndIf
\State $\z \gets \mathcal{E}_\phi(\tilde{\x})$
       \Comment{encode feature}
\If{$\mathrm{rand}() < \tfrac12$}
    \If{$\mathrm{rand}() < \pdropoutz$}
    \State $\z \gets \mu_\z + \sigma_\z \boldsymbol{\eta}$, \; $\boldsymbol{\eta} \sim \mathcal{N}(\mathbf{0}, \mathbf{I})$
           \Comment{replace latent by Gaussian noise}
    \EndIf
\Else
    \State $\forall (\ell, k)\, : \, z^\ell_k \sim \delta\bigl( 1 - \text{Bernoulli}(\pmaskz) \bigr) z^\ell_k$
    \Comment{latent masking}
    
\EndIf

\State \textcolor{blue}{\textbf{// --- Decoder path ---}}

\State sample $t \sim \mathcal{U}(0,1)$ 
    \Comment{masking ratio}
\State sample $\x_t \sim q_t(\x_t \mid \x)$
    \Comment{masked tokens}
\State $\hat{\x} \gets \mathbf{x}_\theta(\x_t, \z)$
       \Comment{compute reconstruction}
\State $\mathcal{L} \gets \mathrm{CE}(\hat{\x}, \x)$ \Comment{cross-entropy versus clean tokens (~\ref{eq:ae})}
\end{algorithmic}
\end{algorithm}

The encoder is warmed up for 1000 steps while the decoder is warmed up for 10000 steps, leaving enough time to kick-start the encoder before the pre-trained decoder is updated. We freeze the decoder's embedding table to further improve training stability.
We finetune our auto-encoder with a batch size of 512 for 200k steps, which takes two days on 16 NVIDIA H100 GPUs.

\subsection{Latent diffusion}\label{sec:details_ldm}

We use the same network architecture for the latent denoiser $\z_\psi$ as in \citet{shabalin2025tencdmunderstandingpropertiesdiffusion}, i.e., a diffusion Transformer \citep{peebles2023scalablediffusionmodelstransformers} using positional noise-level conditioning. 
The learning target is direct data prediction.
We use self-conditioning \citep{chen2023analogbitsgeneratingdiscrete} to inform the network about its past predictions during sampling: we notice that this significantly improves generation quality at virtually no extra computational cost during inference. 
Self-conditioning is enforced during training on 50\% of batches (taken at random) by predicting the network output $\tilde{\z} = \z_\psi(\z_t, t, \emptyset)$ without any conditioning, detaching its gradients, and feeding this conditioning to the network for the actual prediction $\z_\psi(\z_t, t, \tilde{\z})$ which is used for computing the loss.
Pseudo-code for LaDiff training is given in Algorithm~\ref{alg:latent-train}.

\begin{algorithm}[t]
\caption{LaDiff training}
\label{alg:latent-train}
\begin{algorithmic}[1]
\Require encoder $\mathcal{E}_\phi$, latent denoiser $\z_\psi$, schedules $\alpha_t,\sigma_t$,
         latent statistics $(\boldsymbol{\mu}_\z,\sigma_\z)$
\State $\z \gets \mathcal{E}_\phi(\x)$
       \Comment{encode text}
\State $\z \gets (\z - \boldsymbol{\mu}_\z) / \sigma_\z$
       \Comment{normalize latent}
\State sample $t \sim \mathcal{U}(0, 1)$
\State sample $\boldsymbol{\epsilon} \sim \mathcal{N}(\mathbf{0}, \mathbf{I})$
\State $\z_t \gets \alpha_t\,\z + \sigma_t\,\boldsymbol{\epsilon}$
       \Comment{forward diffuse latent}
\If{$\mathrm{rand}() < \tfrac12$}
    \State $\tilde{\z} \gets \z_\psi(\z_t, t, \emptyset)$
           \Comment{self-conditioning prediction}
\Else
    \State $\tilde{\z} \gets \emptyset$
           \Comment{no self-conditioning}
\EndIf
\State $\hat{\z} \gets \z_\psi(\z_t, t, \tilde{\z})$
       \Comment{predict target latent}
\State $\mathcal{L} \gets \lVert \hat{\z} - \z \rVert_2^2$
       \Comment{MSE loss (\ref{eq:gaussian_diffusion})}
\end{algorithmic}
\end{algorithm}

\newcommand{\zlat}{\mathbf{z}}

\begin{algorithm}
\caption{LaDiff sampling}
\label{alg:ladiff-sampling}
\begin{algorithmic}[1]
\Require
  latent schedule $\{\tau_m\}_{m=0}^{M}$,
  discrete schedule $\{t_n\}_{n=0}^{N}$,
  stochasticity control $0 \le \gamma \le 1$,
  latent denoiser $\z_\psi$ 
  discrete decoder $\mathbf{x}_\theta$.

\State \textcolor{red}{\textbf{// --- Latent sampling ---}}
  
\State Sample $\zlat_T \sim \mathcal{N}(\bm{0},\mathbf{I})$
\State $\hat{\zlat} \gets \varnothing$
  \Comment{initialize self-conditioning}

\For{$m = M,\, M{-}1,\, \ldots,\, 1$}
  \If{$\gamma > 0$}
    \State $\tau_{m-1} \gets \sqrt{1-\gamma^2}\,\tau_{m-1}$
      \Comment{warp time grid for $\gamma$ sampling}
  \EndIf
    \State $\hat{\z} \gets \z_\psi\bigl(\zlat_{\tau_m},\, \tau_m,\, \hat{\zlat}\bigr)$
        \Comment{estimate clean latent}
    \State $\hat{\v} = 1 / \sigma_{\tau_m} \bigl( ( \sigma_{\tau_m}\dot{\alpha}_{\tau_m} - \dot{\sigma}_{\tau_m} \alpha_{\tau_m} ) \hat{\z} + \dot{\sigma}_{\tau_m} \z_{\tau_m} \bigr)$
        \Comment{compute velocity}
    \State $\zlat_{\tau_{m-1}} \gets \zlat_{\tau_{m}} - (\tau_{m} - \tau_{m-1}) \hat{\v}$
      \Comment{take integration step}
  \If{$\gamma > 0$}
    \State Sample $\boldsymbol{\epsilon} \sim \mathcal{N}(\bm{0}, \mathbf{I})$
    \State $\zlat_{\tau_{m-1}} \gets \sqrt{1-\gamma^2}\,\zlat_{\tau_{m-1}} + \gamma\,\boldsymbol{\epsilon}$ 
        \Comment{re-noise}
  \EndIf
\EndFor

\State \textcolor{blue}{\textbf{// --- Discrete decoding ---}}
\State Sample $\mathbf{x}_{t_N} \sim \pi_{\mathbf{m}}$
\State $\zlat \gets \sigma_\z \zlat_{\tau_0} + \boldsymbol{\mu}_\z$ 
    \Comment{denormalize latent}
\For{$n = N,\, N{-}1,\, \ldots,\, 1$}
  \State $\hat{\mathbf{x}} \gets \mathbf{x}_\theta^{\ell}\bigl(\mathbf{x}_{t_n},\,\zlat\bigr)$
    \Comment{estimate clean tokens}
  \State $\mathbf{x}_{t_{n-1}} \sim q^\theta_{t_{n-1} \mid t_n}\bigl(\x_{t_{n-1}} \,\big|\, \mathbf{x}_{t_n},\,\x_0 = \hat{\mathbf{x}}\bigr)$
    \Comment{unmask tokens with reverse kernel}
\EndFor
\State $\mathbf{x} \gets \mathbf{x}_{t_0}$
\end{algorithmic}
\end{algorithm}

We use variance-preserving noise schedules as the latents are standardized coordinate-wise using aggregated statistics from the auto-encoder's training.
Following \citet{meshchaninov2026cosmoscompressedsmoothlatent}, we adopt the noise schedule proposed in \citet{hoogeboom2023simplediffusionendtoenddiffusion}:
\begin{align} \label{eq:schedule}
    \mathrm{logSNR}(t) &= -d \log \tan (\pi t / 2) \,, \\
    \sigma^2(t) &= \mathrm{sigmoid}( - \mathrm{logSNR}(t) ) \,, \\
    \alpha^2(t) &= \mathrm{sigmoid}( \mathrm{logSNR}(t) ) \,.
\end{align}

The schedule is visualized for different values of $d$ in Figure~\ref{fig:schedules}, and we perform a grid search for the warping parameter $d$ (see results in Table~\ref{tab:schedule}).
In general, we found $d=10$ to be optimal for most auto-encoder configurations.
This means that, unlike continuous modalities, e.g., natural images or audio, most of the denoising effort should be focused on very noisy latents, which directly relates to the categorical nature of text. 
Indeed, even if we use contextual representations and smoothness regularization, different tokens fundamentally map to discretely distinguishable latent vectors. 
The amount of noise needed to allow a diffusion trajectory to jump between  different latent vectors is thus typically high compared to embeddings of continuous data.

\begin{figure}
    \centering
    \includegraphics[width=0.8\linewidth]{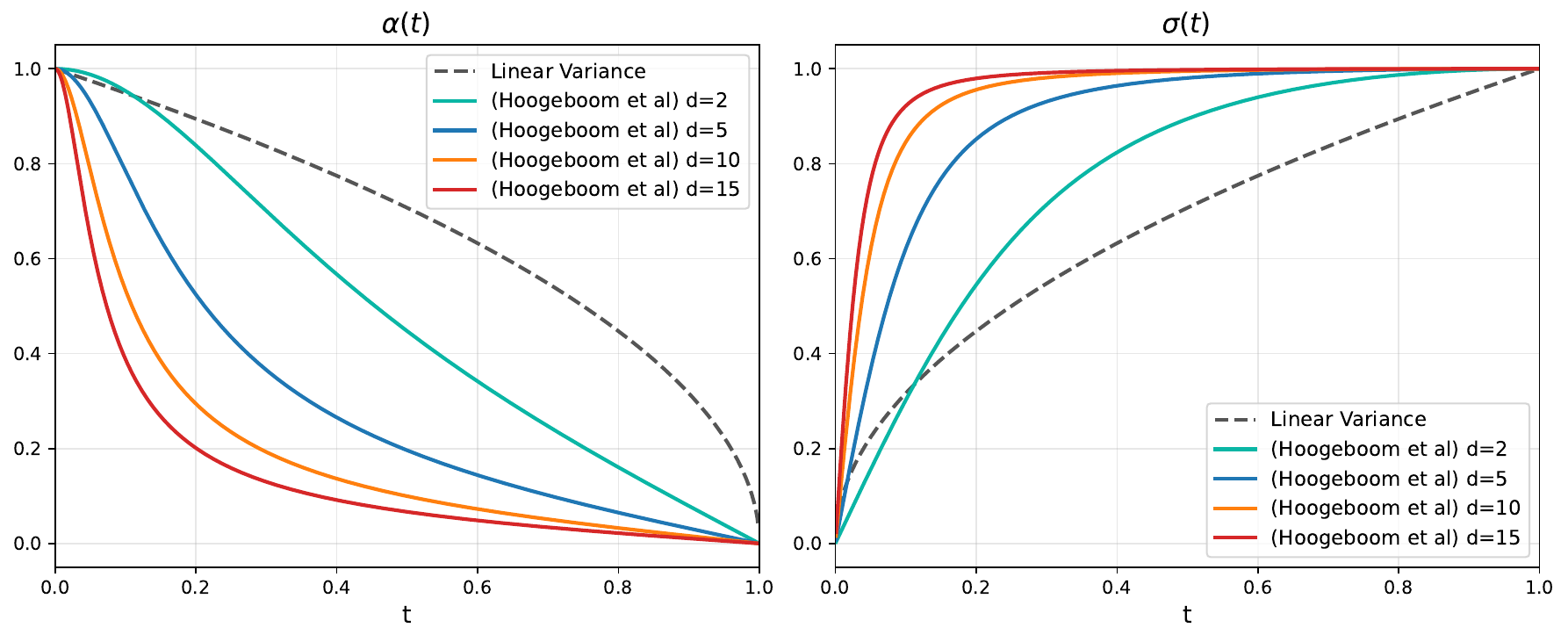}
    \caption{Diffusion schedules following the shape in \citet{hoogeboom2023simplediffusionendtoenddiffusion}. Higher $d$ parameters concentrate timesteps toward higher noise levels.}
    \label{fig:schedules}
\end{figure}

We train our latent diffusion model with a batch size of 512 for 500k steps, which takes two days on 8 NVIDIA H100 GPUs.

\subsection{Distillation}\label{sec:details_distill}

DiLaDiff's network requires an extra time-embedding layer in order to condition the average velocity prediction $\u_\eta(\z_t, t, r)$ on terminating time step $r$. 
We initialize this layer with random weights, and embed the time difference $t-r$ instead of absolute $r$, following \citet{geng2025meanflowsonestepgenerative}.

Self-conditioning is critical for performance in both LaDiff and DiLaDiff. 
We therefore use the \textit{self-conditioned} LaDiff to compute the teacher target for self-distillation. This only requires an extra forward pass (with detached gradients) at training time.
Furthermore, we re-use the teacher's self-conditioning layer in the student network, and pass self-conditioning as $\z_\eta(\z_t, 0) := \Phi \left( \u_\eta(\z_t, t, t), t \right)$ for 50\% of training examples at random, using the property $\u(\z_t, t, t) = \v(\z_t, t)$. $\Phi(\cdot, \cdot)$ is the affine function that maps instantaneous velocities to clean data:
\begin{equation} \label{eq:phi}
    \z = \Phi(\v(\z_t, t), t) := \frac{\sigma_t \v(\z_t, t) - \dot{\sigma}_t \z_t }{\sigma_t \dot{\alpha}_t - \dot{\sigma}_t \alpha_t} \,,
\end{equation}
This expression can be easily derived from the diffusion schedule
\begin{equation}\label{eq:diff}
\z_t = \alpha_t \z + \sigma_t \boldsymbol\epsilon \quad \text{with} \quad \boldsymbol\epsilon \sim \mathcal{N}(0, \mathbf{I}) 
\end{equation}
and definition of instantaneous velocity
\begin{align}
\v(\z_t, t) &= \frac{\mathrm{d}\z_\tau}{\mathrm{d}\tau}\Big|_{\tau=t} \\
 &= \dot{\alpha}_t \z + \dot{\sigma}_t \boldsymbol\epsilon \\
 &= \dot{\alpha}_t \z + \dot{\sigma}_t \frac{\z_t - \alpha_t \z}{\sigma_t} \\ 
 &= \frac{ \left( \dot{\alpha}_t \sigma_t - \dot{\sigma}_t\alpha_t \right) \z + \dot{\sigma}_t\z_t}{\sigma_t} \,,
\end{align}
where the third equality comes from solving for $\boldsymbol{\epsilon}$ in (\ref{eq:diff}).
Solving for $\z$ in the last equality concludes the proof of (\ref{eq:phi}).
At sampling time (see full procedure in Algorithm~\ref{alg:diladiff-sampling}), we dedicate an extra forward pass to compute this self-conditioning.
We also experiment with more efficient designs, where we duplicate the last DiT's MLP head and train it to estimate $\z_\eta(\z_t, 0)$ or $\u_\eta(\z_t, t, 0)$ directly. However, learning this objective while simultaneously optimizing the MeanFlow loss represents an arduous task compared to the limited capacity overhead, and ultimately yielded suboptimal results (see ablations results in Appendix~\ref{sec:additional_distill}).
We dedicate more exploration of this efficient design in future work.

Training pseudo-code is given in Algorithm~\ref{alg:latent-meanflow}.
In practice, the second term in $\u_\text{tgt}$ is efficiently computed using the true Jacobian vector product $
\Big( \partial_\z \u_\eta \; \partial_t \u_\eta \: \partial_r \u_\eta
\Big)^T
\Big( \mathbf{v}^{\top} \; 1 \; 0 \Big)
$.
We sample $t$ and $r$ in the same logit-normal distribution $\text{LogitNormal}(-1, 1)$, where samples are taken from a Gaussian $\mathcal{N}(-1, 1)$ and mapped to $[0, 1]$ using the logistic function. 
At random with 25\% chance, $r$ is taken equal to $t$, thereby reducing to classical flow matching, which improves training stability by matching the student and teacher predictions.
We use loss normalization with a regularization factor of $5$, effectively computing $\Delta / (||\Delta||_2 + 5)$ where $\Delta$ is the loss.
We use a linear warmup of 10k iterations for the tangent term in the MeanFlow objective for improved stability, following \citet{geng2025meanflowsonestepgenerative}.
We self-distill a LaDiff teacher into DiLaDiff with a global batch size of 2048 for 25k steps, which takes 15 hours on 16 NVIDIA H100 GPUs.

\begin{algorithm}[h]
\caption{DiLaDiff training: self-distilling LaDiff with MeanFlow}
\label{alg:latent-meanflow}
\begin{algorithmic}[1]
\Require Encoder $\mathcal{E}_\phi$, teacher $\z_\psi$, student $\mathbf{u}_\eta$,
         schedules $\alpha_t,\sigma_t$, latent statistics $(\boldsymbol{\mu}_\z,\sigma_\z)$
\State $\z \gets \mathcal{E}_\phi(\x)$
       \Comment{encode text}
\State $\z \gets (\z - \boldsymbol{\mu}_\z) / \sigma_\z$
       \Comment{normalize latent}
\State sample $t \sim \text{LogitNormal}(-1, 1)$
\If{$\mathrm{rand}() < \tfrac14$}
    \State $r=t$ 
        \Comment{pure flow matching}
\Else
    \State sample $\sim \text{LogitNormal}(-1, 1)$ and arrange timesteps s.t. $r<t$
\EndIf
\State sample $\boldsymbol{\epsilon} \sim \mathcal{N}(\mathbf{0}, \mathbf{I})$
\State $\z_t \gets \alpha_t\,\z + \sigma_t\,\boldsymbol{\epsilon}$
       \Comment{forward diffuse latent}
\State $\tilde{\z}_\text{teacher} \gets \z_\psi(\z_t, t, \emptyset)$
       \Comment{compute teacher self-conditioning}
\State $\z_\text{teacher} \gets \z_\psi(\z_t, t, \tilde{\z}_\text{teacher})$
       \Comment{teacher latent prediction}
\State $\mathbf{v}_\text{teacher} \gets \dot{\alpha}_t\,\z_\text{teacher} + \dot{\sigma}_t\,\boldsymbol{\epsilon}$
       \Comment{teacher velocity (instantaneous latent velocity field)}

\If{$\mathrm{rand}() < \tfrac12$}
    \State $\tilde{\z}_\text{student} \gets \Phi(\u_\eta(\z_t, t, r, \emptyset) , t)$
        \Comment{update student self-conditioning}
\Else 
    \State $\tilde{\z}_\text{student} \gets \emptyset$
        \Comment{disable student self-conditioning}
\EndIf

\State $\hat{\u} \gets \mathbf{u}_\eta(\z_t, t, r, \tilde{\z}_\text{student})$
       \Comment{student average-velocity prediction}
\State $\u_\text{JVP} \gets \texttt{jvp}\Big( \hat{\u}, (\mathbf{v}_\gamma^T \; 1 \; 0) \, \Big)$ 
    \Comment{tangent term}
\State $\mathbf{u}_{\text{tgt}} \gets \mathbf{v}_\gamma - (t-r) \, \u_\text{JVP}$
       \Comment{MeanFlow target}
\State $\mathcal{L} \gets \mathbb{E}\big[\lVert \hat{\u} - \mathrm{stopgrad}(\mathbf{u}_{\text{tgt}}) \rVert_2^2\big]$
       \Comment{MeanFlow loss}
\end{algorithmic}
\end{algorithm}

\newcommand{\Zflow}{\mathbf{u}_{\eta}}

\begin{algorithm}
\caption{DiLaDiff sampling}
\label{alg:diladiff-sampling}
\begin{algorithmic}[1]
\Require
  latent schedule $\{\tau_m\}_{m=0}^{M}$,
  discrete schedule $\{t_n\}_{n=0}^{N}$,
  stochasticity control $\gamma \ge 0$,
  \texttt{extra\_head} flag,
  self-conditioning threshold $\tau_{\mathrm{thr}}$,
  distilled latent model $\Zflow$ 
  discrete decoder $\mathbf{x}_\theta$.

\State \textcolor{red}{\textbf{// --- Latent sampling ---}}
  
\State Sample $\zlat_T \sim \mathcal{N}(\bm{0},\mathbf{I})$
\State $\tilde{\zlat} \gets \varnothing$
  \Comment{initialize self-conditioning}

\For{$m = M,\, M{-}1,\, \ldots,\, 1$}
  \If{$\gamma > 0$}
    \State $\tau_{m-1} \gets \sqrt{1-\gamma^2}\,\tau_{m-1}$
      \Comment{warp time grid for $\gamma$ sampling}
  \EndIf
  \If{\texttt{extra\_head} \textbf{and} $\tau_m \ge \tau_{\mathrm{thr}}$}
    \State $(\hat{\u},\,\tilde{\zlat}) \gets
      \u_\eta\bigl(\zlat_{\tau_m},\, \tau_m,\, \tau_{m-1},\, \tilde{\zlat}\bigr)$
      \Comment{estimate average velocity + update self-conditioning}
  \Else
    \State $\hat{\u} \gets \u_\eta\bigl(\zlat_{\tau_m},\, \tau_m,\, \tau_{m-1},\,\tilde{\zlat}\bigr)$
      \Comment{only estimate average velocity}
  \EndIf
    \State $\zlat_{\tau_{m-1}} \gets \zlat_{\tau_{m}} - (\tau_m - \tau_{m-1}) \hat{\u}$
        \Comment{take integration step}
  \If{\textbf{not} \texttt{extra\_head} \textbf{or} $\tau_m < \tau_{\mathrm{thr}}$}
    \State $\tilde{\zlat} \gets
      \Phi\Big( \Zflow\bigl(\zlat_{\tau_{m-1}},\, \tau_{m-1},\, \tau_{m-1},\, \tilde{\zlat}\bigr) , \tau_m \Big)$
      \Comment{extra forward pass for self-conditioning}
  \EndIf
  \If{$\gamma > 0$}
    \State Sample $\boldsymbol{\epsilon} \sim \mathcal{N}(\bm{0}, \mathbf{I})$
    \State $\zlat_{\tau_{m-1}} \gets \sqrt{1-\gamma^2}\,\zlat_{\tau_{m-1}} + \gamma\,\boldsymbol{\epsilon}$ 
        \Comment{re-noise}
  \EndIf
\EndFor

\State \textcolor{blue}{\textbf{// --- Discrete decoding ---}}
\State Sample $\mathbf{x}_{t_N} \sim \pi_{\mathbf{m}}$
\State $\zlat \gets \sigma_\z \zlat_{\tau_0} + \boldsymbol{\mu}_\z$ 
    \Comment{denormalize latent}
\For{$n = N,\, N{-}1,\, \ldots,\, 1$}
  \State $\hat{\mathbf{x}} \gets \mathbf{x}_\theta^{\ell}\bigl(\mathbf{x}_{t_n},\,\zlat\bigr)$
    \Comment{estimate clean tokens}
  \State $\mathbf{x}_{t_{n-1}} \sim q_{t_{n-1} \mid t_n}\bigl(\x_{t_{n-1}} \,\big|\, \mathbf{x}_{t_n},\,\hat{\mathbf{x}}\bigr)$
    \Comment{unmask tokens with reverse kernel}
\EndFor
\State $\mathbf{x} \gets \mathbf{x}_{t_0}$
\end{algorithmic}
\end{algorithm}

\clearpage 

\section{Additional experiments}\label{sec:additional}

\subsection{Dependency to masking ratio}\label{sec:additional_masking}

We report in Figure~\ref{fig:reconstruction} the reconstruction performance of our auto-encoder (with a 2$\times$ compression factor) versus pure masked diffusion, as a function of the masking ratio. 
Unlike MDLM, the auto-encoder is able to reconstruct the sequence reasonably well even when the masked input to the decoder has little to no unmasked context. 
Furthermore, we plot in Figure~\ref{fig:attention} the contributions of the cross-attention layers attending to the latent in the decoder. In the two cross-attention layers injected in the decoder (at the first and last position), the contribution of the cross-attention weights increases with the masking ratio, while the contribution of self-attention immediately preceding these cross-attention layers decreases with the masking ratio.

This demonstrates that the latent space contains most of the information needed, and therefore support our claim that the latent is particularly useful at the beginning of generation where available context is scarce.

\begin{figure}[h]
    \centering
    \includegraphics[width=0.9\linewidth]{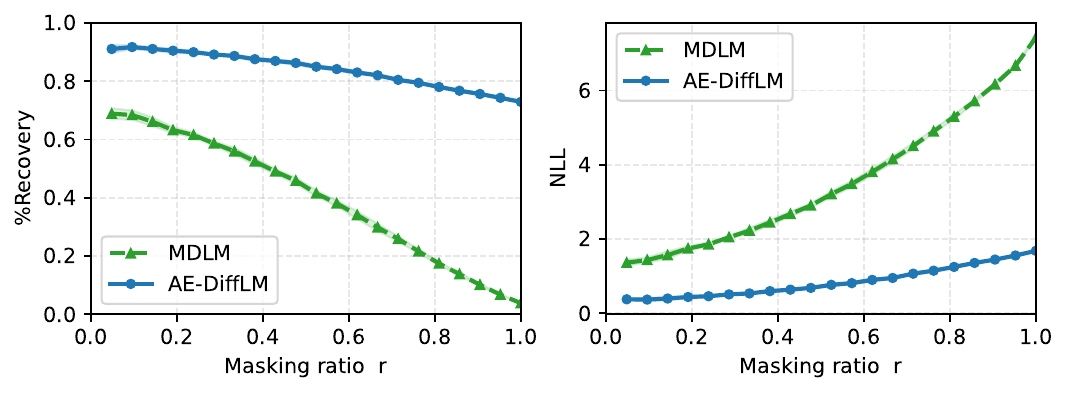}
    \caption{Performance of MDLM and auto-encoder as a function of the masking ratio. Left: Token recovery rate. Right: Negative log likelihood.}
    \label{fig:reconstruction}
\end{figure}

\begin{figure}[h]
    \centering
    \includegraphics[width=0.9\linewidth]{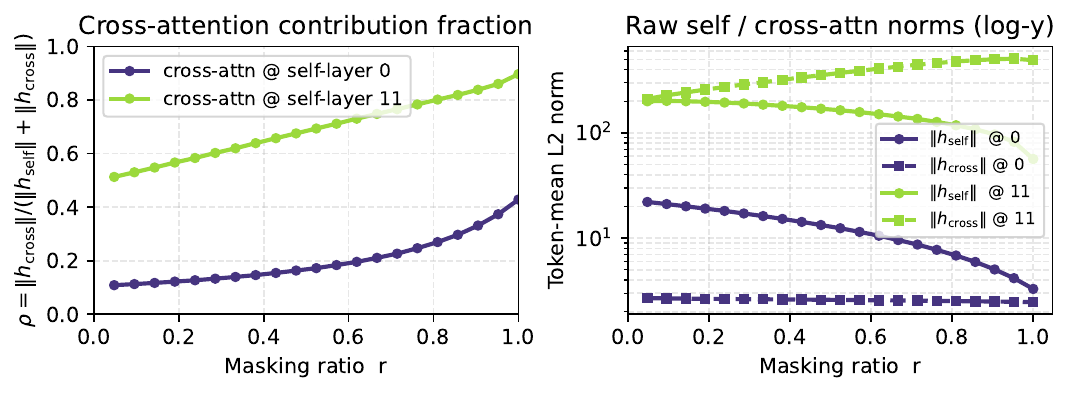}
    \caption{Cross- and self-attention contributions in decoder. Left: Fraction of the cross-attention over self-attention activations. Right: Raw norms of cross- and self-attention in first and last layer.}
    \label{fig:attention}
\end{figure}

\subsection{Confidence-based token selection}\label{sec:additional_confidence}

Specific to diffusion-based language modeling is the ability to use model heuristics to determine the order of token unmasking \citep{nie2025largelanguagediffusionmodels, benhamu2025acceleratedsamplingmaskeddiffusion}. 
One common such heuristic is the model's own confidence, and it can be used to select either the top-k most confident tokens , or all tokens whose confidence exceeds a pre-determined threshold, to be decoded in the current step. 
We present the performance yielded by top-k confidence-based token selection with LaDiff and MDLM in Figure~\ref{fig:confidence}.
Similar to the experiment on temperature-based logits rescaling in Section~\ref{sec:exp_ae}, LaDiff yields reasonable results when applying top-k confidence-based token selection, although the resulting sample entropy decreases severely from $\approx$5.40 down to $\approx5.00$, and MAUVE score suffers significantly from this loss of diversity.
In comparison, MDLM cannot exploit its own confidence and ends up in the trap of repeating very confident tokens such as "the", ".", etc. This yields catastrophic entropy and inexploitable results.
Importantly, we dot not mean to imply that these results do simply transfer at scale: model confidence is de-facto a reliable heuristic for larger purely discrete diffusion models such as \citet{nie2025largelanguagediffusionmodels}.
However, we demonstrate that, at the reference GPT2-small scale, latent guidance unlocks the ability of the discrete decoder to exploit its own confidence.

\begin{figure}[h]
    \centering
    \includegraphics[width=\linewidth]{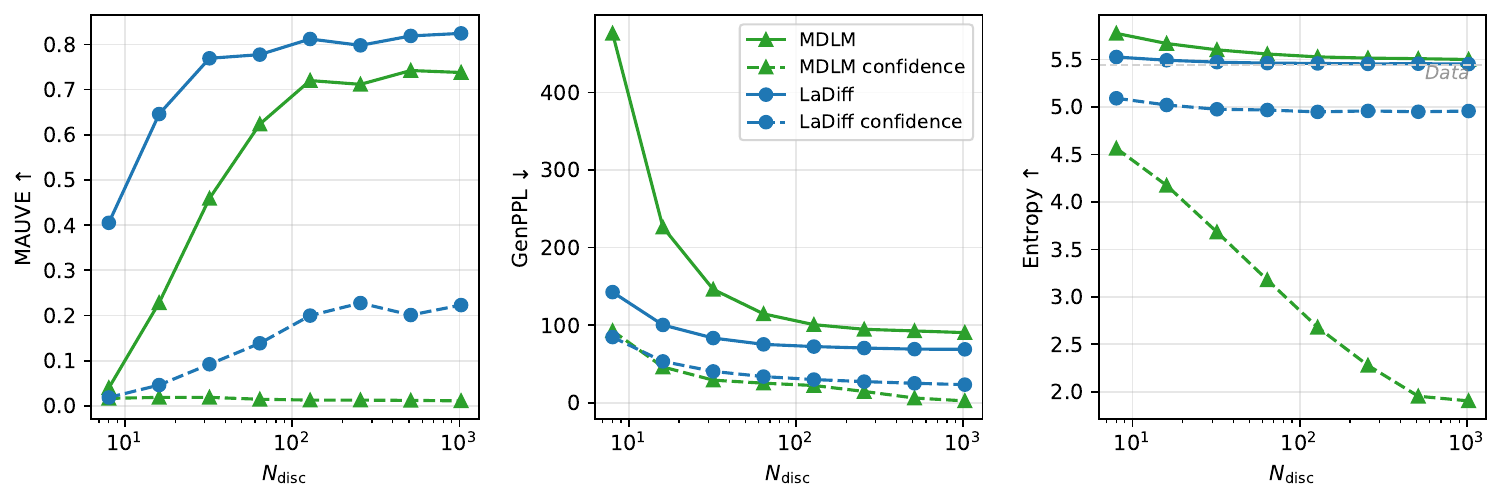}
    \caption{Confidence-based vs random token selection. $\Nc=200$.}
    \label{fig:confidence}
\end{figure}

\subsection{Decoder robustness and latent diffusion schedule}\label{sec:additional_ae}

We visualize the robustness of our decoder to noisy text latents, for a fixed set of latent augmentations, in Figure~\ref{fig:robustness}. As already demonstrated in Table~\ref{tab:ldm_ae}, larger latent spaces have better token recovery rates. We further show here that they exhibit similar robustness profiles, i.e. that the decoder is very robust to noise levels below a standard deviation of 0.8, above which the token recovery rate drops steeply.
As the transition region is located toward $\sigma\approx0.8$, we skew the latent diffusion schedule toward high noise levels so as not to waste training budget on very easy denoising instances.
We therefore tune the $d$ parameter of the variance-preserving schedule in (\ref{eq:schedule}). Results are displayed in Table~\ref{tab:schedule}, and as in most preliminary experiments the grid-search converged toward an optimal parameter of $d=10$.

\begin{minipage}[c]{0.48\textwidth}
  \centering
  \includegraphics[width=\linewidth]{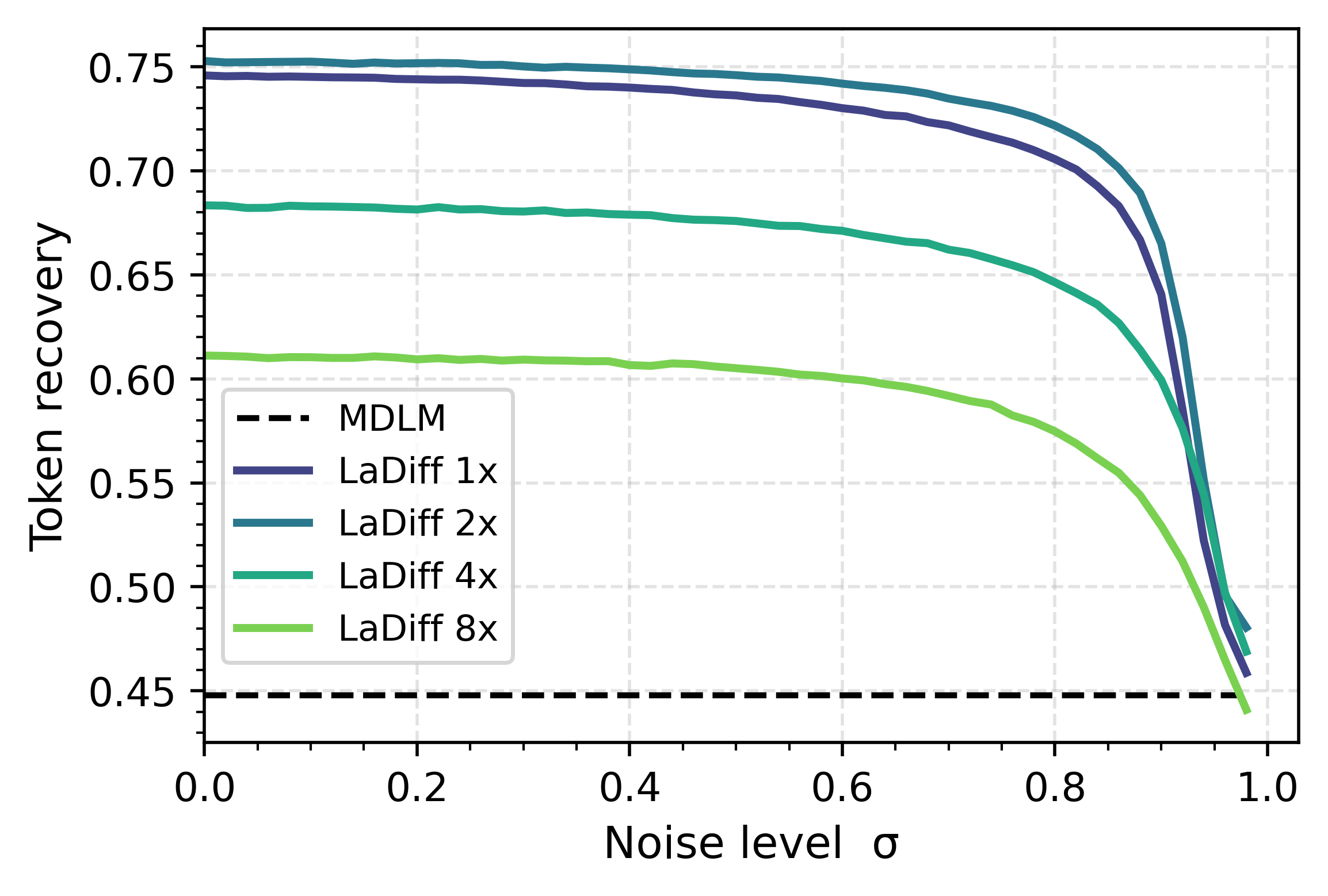}
  \captionof{figure}{Decoder robustness to latent noise on OpenWebText: MDLM and LaDiff with different compression factors.}
  \label{fig:robustness}
\end{minipage}\hfill
\begin{minipage}[c]{0.5\textwidth}
    \centering
    \scalebox{0.8}{
    \begin{tabular}{l c ccc}
        \toprule
        & $d$ & GenPPL ($\downarrow$) & Entropy ($\uparrow$) & MAUVE ($\uparrow$) \\
        \midrule 
        \textit{Data} & - & 14.8 & 5.44 & 1.00  \\
        \midrule
         LaDiff & 2 & 92.0 & \textbf{5.52} & 0.78 \\
         LaDiff & 5 & 71.4 & 5.45 & 0.80 \\
         LaDiff & 10 & \textbf{62.3} & 5.40 & \textbf{0.82} \\
         LaDiff & 15 & 75.4 & 5.45 & 0.81 \\
         \bottomrule
    \end{tabular}
    }
    \captionof{table}{Generative performance of LaDiff when varying the noise schedule parameter $d$ in~(\ref{eq:schedule}). $\Nc=200, \Nd=1024$.}
    \label{tab:schedule}
\end{minipage}

We propose an alternative, more principled approach, following the discussion in \citet{lee2026flowmaplanguagemodels}, where a closed-form of the token decoding error rate is available in their one-hot encoding setup.
The decoding error rate $\omega(t)$ is defined as the error rate in the token space, for a linearly-evolving noise variance $\sigma^2(t) := t$:
\begin{equation}
\omega(t) = 1 -  \frac{\%\mathrm{Recovery}(t)}{\%\mathrm{Recovery}(0)} \,.
\end{equation}
We do not have a closed-form expression of $\omega(t)$ as in \citet{lee2026flowmaplanguagemodels}, therefore we perform a fit of the empirical recovery rate in Figure~\ref{fig:robustness} to derive a diffusion time step reparameterization.
We find that hyperbolic tangent functions can provide a satisfying fit of the decoding error rate function, with only four parameters:
\begin{equation}
    \omega_\mathrm{fit}(t ; \{ k, t_0, \omega_\mathrm{min}, \omega_\mathrm{max} \} ) := \omega_\mathrm{min} + (\omega_\mathrm{max} - \omega_\mathrm{min}) \frac{1 + \tanh (k * (t - t_0)) }{2} \,
\end{equation}
which we then invert to obtain the reparameterized time:
\begin{equation}
    \tilde{t} = \omega_\mathrm{fit}^{-1}(\omega ; \{ k, t_0, \omega_\mathrm{min}, \omega_\mathrm{max} \}) := 
    t_0 + \frac{1}{k} \mathrm{atanh}(2 * \frac{\omega - \omega_\mathrm{min}}{\omega_\mathrm{max} - \omega_\mathrm{min}} - 1) \,.
\end{equation}

We then treat $\tilde{t}$ as our diffusion time parameterizing a linear variance-preserving schedule:
\begin{align}
\sigma^2(\tilde{t}) &= \tilde{t} \\ 
\alpha^2(\tilde{t}) &= 1 - \tilde{t} 
,.
\end{align}
We visualize the fit in Figure~\ref{fig:der_vp} and the resulting time reparameterization in Figure~\ref{fig:der_vp_tau}.

\noindent
\begin{minipage}[t]{0.48\textwidth}
  \centering
  \includegraphics[width=\linewidth]{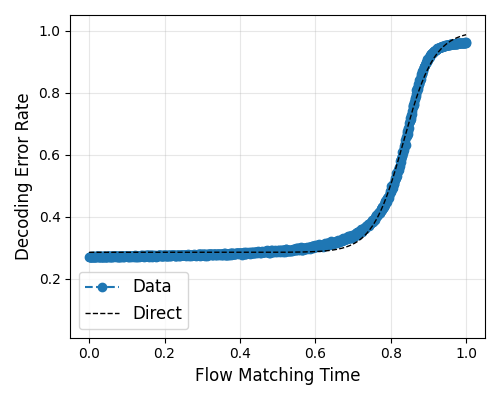}
  \captionof{figure}{Empirical decoding error rate and corresponding sigmoid-tanh fit $\omega_\mathrm{fit}(t)$. 
    }
  \label{fig:der_vp}
\end{minipage}
\hfill
\begin{minipage}[t]{0.48\textwidth}
    \centering
    \includegraphics[width=\linewidth]{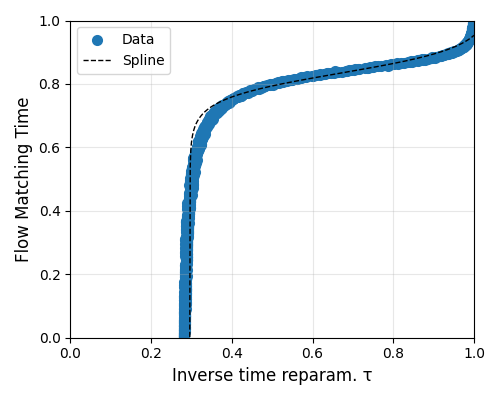}
    \captionof{figure}{Time repararameterization $\tilde{t}(\omega)$}
    \label{fig:der_vp_tau}
\end{minipage}

\subsection{Distillation}\label{sec:additional_distill}

\paragraph{Distillation method}

We tried MeanFlow \citep{geng2024consistencymodelseasy} as well as TVM \citep{zhou2026terminalvelocitymatching} for self-distilling LaDiff into DiLaDiff but did not notice significant differences in our experiments. 
We also experimented with different mean and standard deviation for the LogitNormal distribution used to sample $(t, r)$.
We ended up using MeanFlow with LogitNormal parameters $P_\text{mean}=-1, P_\text{std}=1$.
For completeness, performance is reported in Figure~\ref{fig:meanflow_tvm}.

\begin{figure}[h]
    \centering
    \includegraphics[width=\linewidth]{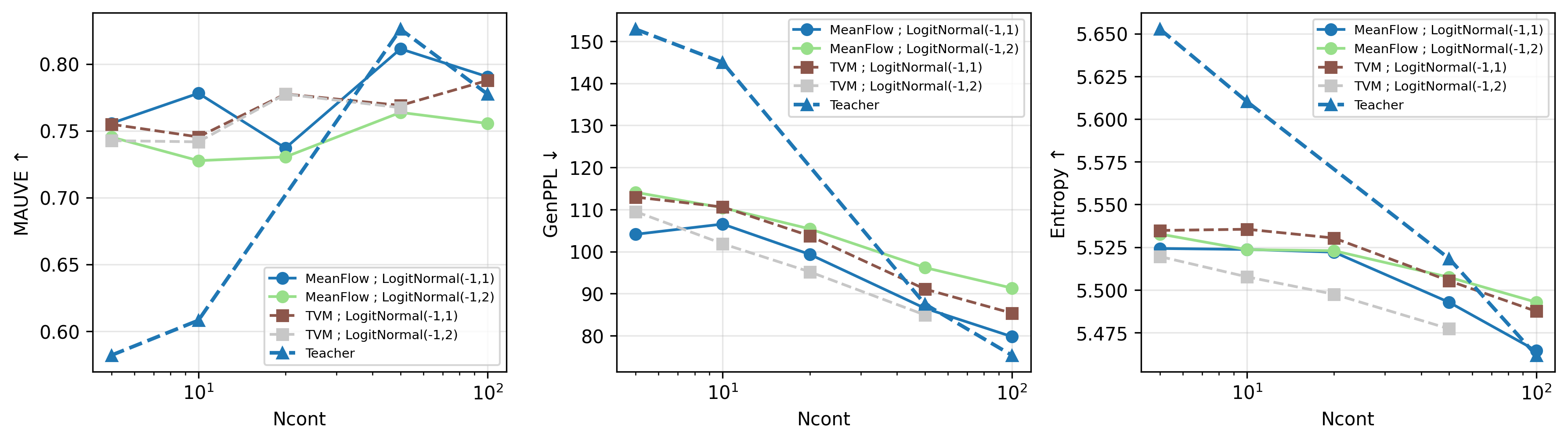}
    \caption{Ablations of self-distillation method for DiLaDiff. $\Nd=64$.}
    \label{fig:meanflow_tvm}
\end{figure}

\paragraph{$\gamma$-sampling}

We ablate the choice of $\gamma$ for stochastic-controlled sampling \citep{kim2024consistencytrajectorymodelslearning} in our few-step DiLaDiff model. The parameter $\gamma \in [0, 1]$ controls the amount of noise re-injected into the prediction:
\[
\left(
  \tau_m \overset{\text{Denoise}}{\longrightarrow}
  \sqrt{1 - \gamma^2} \tau_{m-1} \overset{\text{Noisify}}{\longrightarrow}
  \tau_{m-1}
\right)_{m=1}^{M}
\]

Using $\gamma=0$ yields the deterministic ODE sampler while $\gamma=1$ yields consistency-style stochastic sampling (at each step, jump all the way to clean data then compute the forward diffusion kernel to the next time step $\tau_{m-1}$).
Results for DiLaDiff with $\Nc=5, \Nd=32$ are plotted in Figure~\ref{fig:gamma}.
Surprisingly, entropy decreases when $\gamma$ increases, up to $\gamma=0.8$, above which GenPPL explodes. $\gamma=0.8$ seems to yield the best optimal point for MAUVE, although deviations are minimal compared to MAUVE's noise floor.

\begin{figure}[h]
    \centering
    \includegraphics[width=\linewidth]{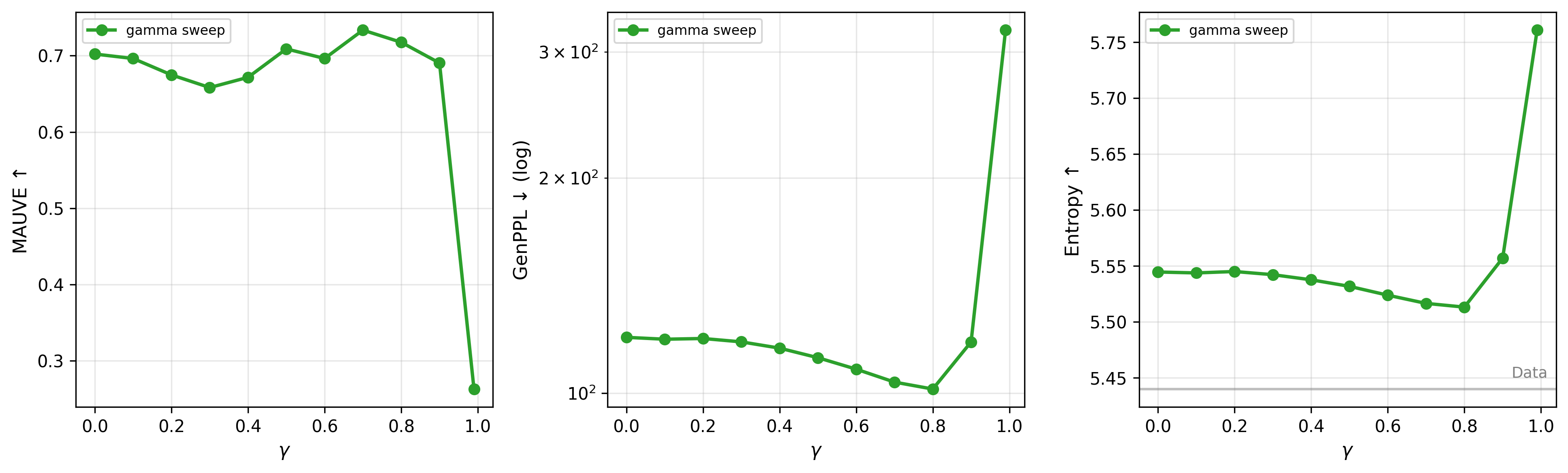}
    \caption{$\gamma$-sampling with DiLaDiff. $\Nc=5, \Nd=32$.}
    \label{fig:gamma}
\end{figure}

\paragraph{Self-conditioning}

We show here our experiments regarding self-conditioning design in DiLaDiff. 
When using an extra MLP head to estimate the self-conditioning input during training, we can either use that prediction at each step, and therefore obtain self-conditioning for free, or compute an extra forward pass to compute $\z_\eta(\z_t, 0) := \Phi\left(\u_\eta(\z_t, t, t), t\right)$, cf. Section~\ref{sec:details_distill}.
We observed that the more sophisticated design using the extra MLP head obtained worse results, which can be due to two independent factors: (\textit{1}) the more challenging training objective, as the student network is being asked to optimize both the MeanFlow loss and the extra regression loss $|| \z_\eta(\z_t, 0) - \z ||_2^2$ using the same backbone architecture; (\textit{2}) the fundamental limitation of previous-step self-conditioning at low NFEs regime.
The latter issue exists for both normal and distilled models, and comes from the mismatch between (\textit{1}) training, where self-conditioning is computed as $\tilde{\mathbf{z}} = \z_\psi(\z_t, t, \emptyset)$ and used to refine the estimation $\z_\psi(\z_t, t, \tilde{\mathbf{z}})$ and (\textit{2}) inference, where the self-conditioning is obtained at the current step $\tilde{\mathbf{z}} = \z_\psi(\z_{t\tau_{m}}, \tau_m, \tilde{\mathbf{z}}_\text{previous})$ and used in the \textit{next step} to compute $\z_\psi \left( \z_{t\tau_{m-1}}, \tau_{m-1}, \tilde{\mathbf{z}} \right)$. In many-step regime, the steps are sufficiently close such that this inference-training mismatch is reasonable, however for few-step generation, this mismatch cannot be neglected.
We show results for the extra-MLP approach, compared to the normal DiLaDiff, on Figure~\ref{fig:vhead}.

\begin{figure}[h]
    \centering
    \includegraphics[width=\linewidth]{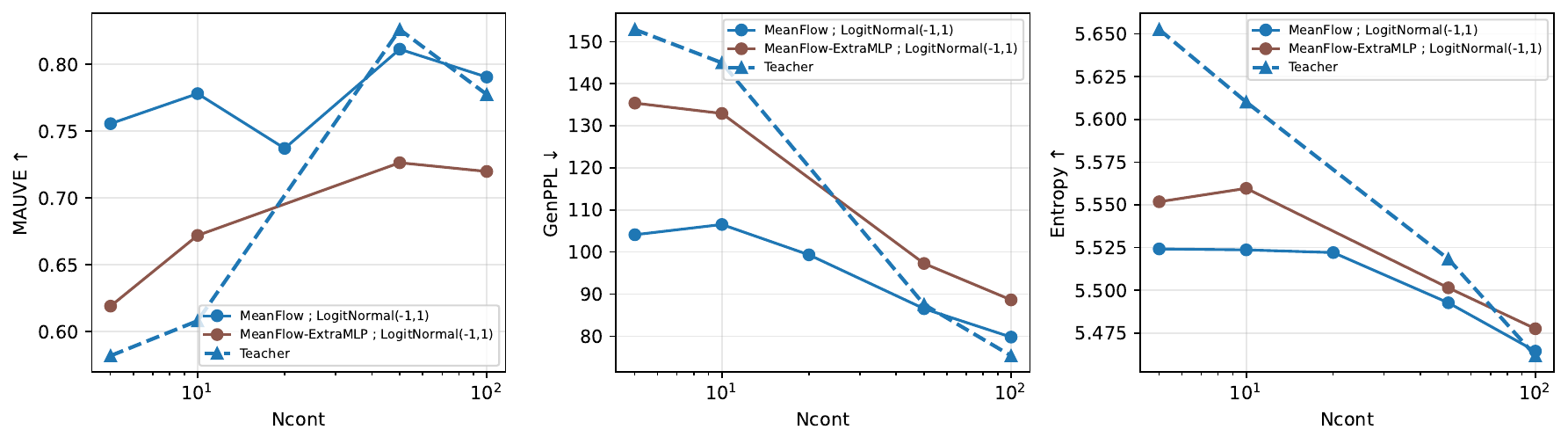}
    \caption{DiLaDiff with/without extra-MLP head for predicting self-conditioning. $\Nc=5, \Nd=64$.}
    \label{fig:vhead}
\end{figure}

\newpage

\section{Baselines}

For various reasons (existing and current state of implementation, choice of the BERT embedding as encoder feature, etc.), we end up using the experimental setup by \citet{meshchaninov2026cosmoscompressedsmoothlatent}, while \citet{sahoo2024simpleeffectivemaskeddiffusion, sahoo2025diffusionduality} use the setup in \citet{sahoo2024simpleeffectivemaskeddiffusion}. 
For transparency, the main differences between the respective setups are listed in Table~\ref{tab:setup_comparison},
and mostly result from the different DiT implementation and different tokenizer (influencing the number of embedding and logit head parameters).

\begin{table}[h]
    \centering
    \scalebox{0.8}{
    \begin{tabular}{lc|c}
        \toprule \midrule
        & Ours & \citet{sahoo2024simpleeffectivemaskeddiffusion} \\
        \midrule
         Architecture & \makecell{Absolute positional embedding \\ QK norm} & \makecell{RoPe \\ (unused) AdaLN} \\
          Tokenizer & \texttt{bert-base-uncased} & \texttt{gpt2-tokenizer} \\
         Vocabulary size & 30, 522 & 50, 257 \\
          Embedding / Logit parameters & 22M & 38M \\
          Transformer parameters & 84M & 92M \\
          Total parameters & 136M & 170M \\
         Denoiser post-processing & Set mask logit to \texttt{-inf} & \makecell{ Set mask logit to \texttt{-inf}. \\ Manually force logits of unmasked tokens to $0$ \\ (for corresponding logit) and \texttt{-inf} (for other logits)} \\
    \end{tabular}
    }
    \caption{\protect\centering Differences between original MDLM baseline and our re-implementation.}
    \label{tab:setup_comparison}
\end{table}

\section{Pareto frontier for batch size 1}

\begin{figure}[H]
    \centering
    \includegraphics[width=\linewidth]{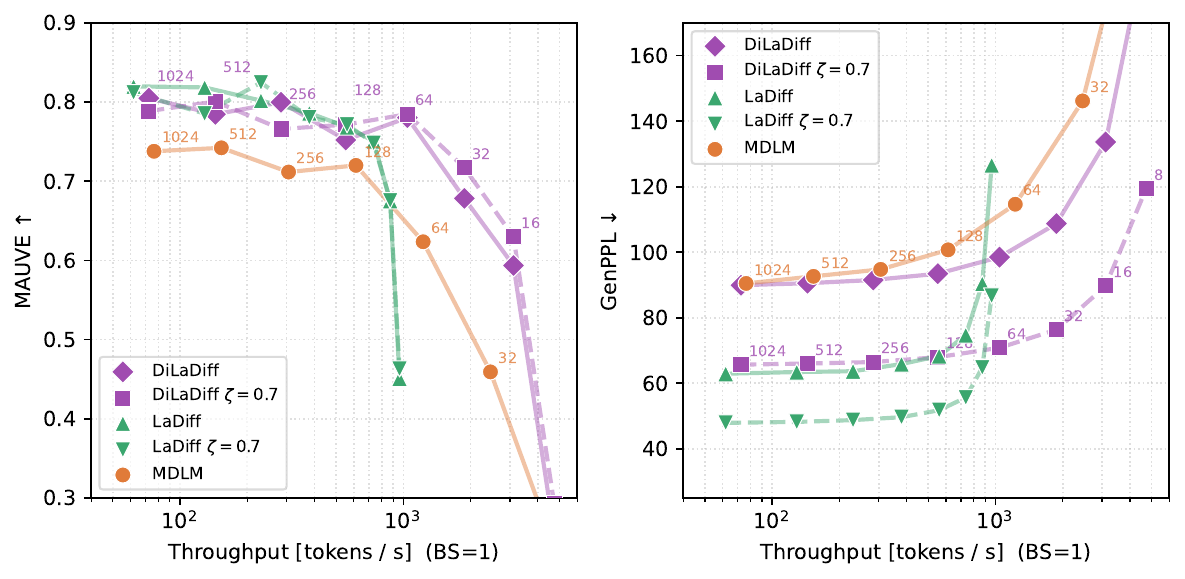}
    \caption{Speed-quality Pareto frontier for batch size $\mathrm{BS} = 1$.}
    \label{fig:pareto1}
\end{figure}

\section{Likelihood computation}
\label{sec:ppl_evaluation}

\subsection{Auto-encoder}
\label{sec:ppl_ae}

We compute perplexity of our auto-encoders using the following expression:

\begin{equation}
\log p(\x) \leq \mathcal{L}_\text{ELBO}^\text{AE} := \mathcal{L}_{\text{mELBO}}^{\text{AE}} + \mathcal{L}_{\mathcal{E}} \,.
\end{equation}

The first term is the modified ELBO corresponding to our training objective (\ref{eq:ae}) with the addition of the original weighting term $\displaystyle{\frac{\dot{\alpha}_t}{1 - \alpha_t}} = -\frac{1}{t}$ in equation (11) in \citet{sahoo2024simpleeffectivemaskeddiffusion}. In practice, we use the approach by \citet{nie2025largelanguagediffusionmodels} which replaces the continuous $\frac{1}{t} \, , t \sim \mathcal{U}(0,1)$ with a discrete $\frac{L}{k}\, , k \sim \mathcal{U}(\{ 1, \dots, L \})$, thereby increasing stability:
\begin{equation} \label{eq:ppl_ae}
    \mathcal{L}_{\text{mELBO}}^{\text{AE}} = 
    \mathbb{E}_{\substack{k \sim \mathcal{U}(\{1, \dots, L \}) \\ \x_k \sim \tilde{q}_k( \x_k | \x )}} 
    \displaystyle{\frac{L}{k}}
\sum_{\ell} 
    \delta\left(\x_k^\ell = M\right)
    \log \langle \x_\theta^\ell(\x_k, \z), \x_k^\ell \rangle \,.
\end{equation}
We slightly abuse notations to signify that the kernel $\tilde{q}_k$ masks exactly $k$ tokens at random positions in $\x$, yielding $\x_k$.
This term is evaluated using 128 Monte Carlo samples, which is enough to obtain a tight and low-variance bound according to \citet{nie2025largelanguagediffusionmodels}.

The encoder entropy vanishes because of our latent standardization procedure, when using uniform dequantization:
\begin{equation}
\mathcal{L}_\mathrm{\mathcal{E}} := \log q^\mathrm{uni}_\phi(\z | \x) = S L \log( \sigma_z ) = 0
\end{equation}

Similarly, we compute the token recovery rate as
\begin{equation}
    \%\text{Recovery} = \mathbb{E}_{\substack{t \sim \mathcal{U}[0,1] \\ \x_t \sim q_t( \x_t | \x )}} 
        \mathbb{E}_{\z \sim \mathcal{E}_\phi (\z | \x) } 
    \mathbb{E}_{\ell} \Big(\delta\left(\x_k^\ell = M\right)  \delta\left(
         \x_\theta^\ell(\x_t, \z) = \xl
        \right) \Big)
\end{equation}

\subsection{LaDiff}
\label{sec:ppl_ladiff}

Although we don't present quantitative results here, we propose the following expression for computing the likelihood of LaDiff:
\begin{equation}
\log p(\x) \leq \mathcal{L}_\text{ELBO}^\text{LaDiff} := \mathcal{L}_\text{PF-ODE}^\text{Latent} + \mathcal{L}_{\text{ELBO}}^{\text{AE}} \,.
\end{equation}
The first term is the true probability-flow ODE likelihood of our latent $\z$:
\begin{equation}
    \mathcal{L}_\mathrm{PF-ODE} := \log p_\theta(\z) = \log p_T(\z_T) + \int_{0}^{T} \nabla \cdot \tilde{\mathbf{z}}_\gamma(\z_t, t) \mathrm{d}t \,,
\end{equation}
which is evaluated according to the protocol in \citet{song2021scorebasedgenerativemodelingstochastic}. The divergence $\nabla \cdot  \tilde{\mathbf{z}}_\gamma(\z_t, t)$ is approximated with help of Hutchinson's trace estimator (with one Rademacher-distributed noise sample) and the PF-ODE is integrated using the Runge-Kutta-45 solver.

The second term is the decoder ELBO in~(\ref{eq:ppl_ae}).

\section{Qualitative samples}

\begin{tcolorbox}[
  enhanced,
  colback=black!3,
  colframe=black!55,
  boxrule=0.6pt,
  arc=2pt,
  boxsep=8pt,
  left=12pt, right=12pt, top=8pt, bottom=12pt,
  toptitle=10pt,
  bottomtitle=8pt,
  title={%
    \textbf{LaDiff}\quad
    $\Nc=200$,\quad $\Nd=1024$\quad
    $\cdot$\quad
    \textbf{GenPPL:}~62.9\quad
    $\cdot$\quad
    \textbf{Entropy:}~5.40
  },
  fonttitle=\small\sffamily,
  coltitle=black,
  colbacktitle=black!10,
  titlerule=0.5pt,
]
\small\setlength{\parskip}{0.45em}\setlength{\parindent}{0pt}
\raggedright

in 18 months after seeing her first deer strike about 100 kilometres from Corbrooke Avenue in Ontario. I think it's a bit cliche to think there are members of the public out there that are actually advocating against animal control. ``It's just a small box, for sure, but it's just a crime. They can't keep them in themselves. They went there, it was horrible, but it was hard to think they could move it.'' Heather McIlgain, director of the Canadian Society of Deer and Hunters, a group that has fought against animal killings, said unusual cases are happening. She said the Canadian government has targeted some wild animals, some of which are heading to Gatehead, and the endangered animals are being being held in animal killing investigations. The association's study, combined with a collaboration with both Bell Region Public Safety and the NPP, found 456 animal sightings over the past year, mostly with government-led governments and provincial agencies that ban animal killing. ``I don't know why the owner has killed an endangered animal,'' Gravelsome said. ``There's no question about the facts. `We are still working out which person is responsible, is this accident or non-profit,' Gravelsome said.'' ``We've seen a lot of wild animals killed again over time. Even if they are not responsible for all animal killing investigations in the country it doesn't even seem to assess the welfare of the animals that are being examined.'' ``The Government of Ottawa is funding assistance to all government agencies and the private sector to close the animal slaughter industry in Canada, fully shut down government offices and securely execute the rescue of farm animals.'' McIlgain said while she hasn't opened a door open an inquiry into how federal agencies are looking at animals, she is also asking for information about livestock and others who are planning on killing the animals. ``We were just looking at farm animals. We saw a bunch of small mags, they came out of somewhere where they had stepped in,'' she said. ``Some of the those animals had infected and turned up in traffic, and I brought up the animal, and then by laying it off, they started to move much bigger people. Then these bull lions came in, a bit bit bigger and made for thousands of clicks.'' She said the woman was seeing a herd of bull lions and a group of animals grazing on deer. ``We pushed them several yards down the highway, so there was a bush that seemed to be circling up around.'' ``Then we had a stump. A cule Without a stump, the cule was hard to spot,'' she said. Once the woman finished up the animal, the deer became injured. ``It's hard to know what to see,'' she said. ``He had a gun and I just held it, but it could easily pick one with me. Eventually the woman grabbed some of the deer and John was able to keep the animals in hand,'' she said. ``LEASEN WORNER / FOX News Post-Morning Nurse Wayne Smith was injured when she was hit by a car in her 100-week nursing home, ABC News spoke to Cun Clhree on Monday, standing motionless at his home, where she also `collided with her deputy dog' on Tuesday morning.'' ``He kept mum about not telling me his his story,'' Clhree said. The personal caretaker told ABC News. ``He's very hurt. In fact, you can't tell him how much damage he has done.'' Doctors are still deciding how exactly that is affected, but her apparently 50{,}000 year-old animal is in competent mental condition. He'll work all days, being bitten, sometimes twice a day of the week. ``(Danielle) is recovering physically and mentally now and he doesn't even get to his care facility. He is sitting and I could have been an accident,'' said Clhree. ``He's just so fragile,'' she said. ``As soon as he feels it, he's not getting injured. He's getting hurt that much and he's back there.'' ``Often, when people feel like someone has hit on something important, it's really not all they are.'' Another child, injured Monday in her 18-week-old nursing car, made the decision courtesy of Post-Morning. The child is unhappy because of her concerns, particularly the parents, who wanted to speak publicly about the issue this week. Students were to attend another school Sunday, about this week as Danielle was very good at her pet bullies and vannalism, and they also applied for custody of her mother. Topics: zoo-reality, zoo-reality, zoo, maaaa-8162, 5078-261, northern-2000, australia First Javascript

\end{tcolorbox}

\clearpage 


\begin{tcolorbox}[
  enhanced,
  colback=black!3,
  colframe=black!55,
  boxrule=0.6pt,
  arc=2pt,
  boxsep=8pt,
  left=12pt, right=12pt, top=8pt, bottom=12pt,
  toptitle=10pt,
  bottomtitle=8pt,
  title={%
    \textbf{LaDiff}\quad
    $\Nc=200$,\quad $\Nd=64$\quad
    $\cdot$\quad
    \textbf{GenPPL:}~68.3\quad
    $\cdot$\quad
    \textbf{Entropy:}~5.40
  },
  fonttitle=\small\sffamily,
  coltitle=black,
  colbacktitle=black!10,
  titlerule=0.5pt,
]
\small\setlength{\parskip}{0.45em}\setlength{\parindent}{0pt}
\raggedright

``and I look forward to seeing that opportunity.'' Season resolved Schumacher, who earned his second finish in the TRT this season, as he won the race at the age of 13. ``I was introduced from the start of the season and had a very bad weather. When I started training, I kicked off at Turn 6 and then dropped off a little but didn't know what else to do. I started training a little bit and was really very excited to race the next race. Now in the TRT is an amazing opportunity to continue in my four years in the RRA.'' ``Ferrari will resume a Formula 5 TRA campaign in the new Rosso team, led by Mauricio Quesadeso (NP) who raced for 24 Spa from 2007 to 2012, winning victories in La Island and 21 Spa in the TR.'' ``I'm very happy with the results, so I'm really looking forward to back racing after the race next year,'' he said of ahead of the 2017 F1 season. ``It's a great time to get the start of the new season, but it brings a value to my team and I hope the will continue to continue in Le Mans.'' Overall, although with a bit of hard work, Ferrari returns to the form and form with France-based \"{O}rrama Ferrari FTS---TRT, \$10{,}000 (TRAMA MEDIO). His 4th podium finish in the TRT Championship this season Total Shares \$514{,}419 Goldman Sachs (Deutsche Bank ETF ET) SWIFT---SD) \$39{,}599. General Counsel (NASDA), (OR) \$524{,}614. Morgan Investments (MSFT) Goldman Sachs (AIRN) 2, 16 Chasegan Morgan Co (DHL) 12 Goldman Sachs Group (TB) 13. Morgan Stanley Capital Management (JPIF) 14 JPMorgan Chase Company (JPY) 17, 23. General Management (MSGI) 20 SwissGrann (SPDR) Wells Fargo Financial Management (FNCE) 47, 17 Goldman Sachs Accultatim Future (FINAL,) \$40{,}000--\$55{,}100{,}000. Citi Asset Management (BBBM) 1, 21 Wall Street Financial \& Dividends Inc. \$100{,}003. Morgan Stanley INK (DAK) 13. Fargo ETF (TES) 18 Morgan Stanley \& Co BYC Companies (TTX) 20 JPMorgan Chase Accultati ((ADK) 13, 25. Bank of America (U.P.) 24 JP Morgan Chase Morgan Co (TB) 26 Bank of America Group (CSA) Alphabet Holding (TAG) 18, 28 Goldman Sachs Group (GK) Morgan Stanley (TK) 24 Morgan Stanley Morgan Group (TLD) 22 Morgan Stanley Global Markets (TE) 23. Morgan Stanley Banking Group (TA) 23 Standard TLDs Vanguard Goldman Morgan Group (MS) 25 Goldman Sachs Group (I) JPMorgan Morgan Co (INC) + 9, 49 JP Morgan Chase \& Co (TLDF) 9, 29 The Bank of America ((BUDAQ)) Goldman Sachs \& Co (NYSE: M) Then again, we have more data in data from the past results. The 23 stocks are by a third. The 23 funds are frontier-fund funds. Here are mature funds. We show how many companies have signed their companies as member companies. Goldman Sachs American Life Inc (NASDAQ) TLDs. First, we show ``pioneer-pay'' funds. (Jeffrey Ackman, DLOM) Morgan Stanley ETF + ((AD)) Liberty Mutual Inc (METR) Mark Morris America, President of Alpha Foundation GEE ``7.5\% of Morgan Stanley is the largest investment fund,'' as we noted in the post-Alpha US Investment Fund report. We've looked for investor funds in each index based on the most recent hedge funds. Secondly, we've only used this parameter in our initial analysis (published under SPF). Thirdly, we see ``pioneer--pay'' stock funds. Next, we look at those ``legacy'' funds. We name those funds. So here are the patterns for investors---for some reason: 1) There's 150 investors plus a top tier investor per step vs.\ algorithm. Our algorithm carries a fixed list price. Next, we choose confidence in trading. Since moving down, some investors will move lower. Weeks 2, 3, 5 ``move down,'' stocks are more than 50, but by Weeks 30, investors move up. So, while we've picked the SPF over the last 4 years, we don't want investors to agree on our data and the SPF shows a good, optimistic performance. Below is the time, after Vanguard Capital Partners launched

\end{tcolorbox}

\clearpage 


\begin{tcolorbox}[
  enhanced,
  colback=black!3,
  colframe=black!55,
  boxrule=0.6pt,
  arc=2pt,
  boxsep=8pt,
  left=12pt, right=12pt, top=8pt, bottom=12pt,
  toptitle=10pt,
  bottomtitle=8pt,
  title={%
    \textbf{MDLM}\quad
    $\Nd=1024$\quad
    $\cdot$\quad
    \textbf{GenPPL:}~90.5\quad
    $\cdot$\quad
    \textbf{Entropy:}~5.50
  },
  fonttitle=\small\sffamily,
  coltitle=black,
  colbacktitle=black!10,
  titlerule=0.5pt,
]
\small\setlength{\parskip}{0.45em}\setlength{\parindent}{0pt}
\raggedright

I was investing heavily in gold due to my digital investment until recently, now many of the biggest stock bubbles investing in the world could afford to purchase. Back then we were willing to sell and to use them to buy coins. Though we did have a lot of people buying Gold coins, there are some big problems that changed people's minds. Some new crypto-trading contracts were invented back in the day there were short amounts to short selling cycles, the BO holders get much lower price than Bitcoin price. I bet most today we are very interested in coins. I assume all of cheap stuff instantly. When we invest in fiat today we got very high, but at the lowest price in any other priced piece of history. You would give yourself time to ask why gold today. Never have you be surprised if you bought gold. Have you ever invested \$1 in fiat? It would be worth it if you most likely will never buy it. Today we are only in 30-something and many feel so about it. Often it counterfeit sell instantly, but value can be much higher than there ever. One of the biggest issues I discuss is the possibility of giving the money at a central bank judging the quality of the virtual currency price. I do not believe the cryptocurrency is the point to which we understand and buy makes a lot of tokens. This is true for any silver we have at all. A lot of that Commo-pricing To put it another way, a savings of \$9.99 billion in annual costs per gold by making the US economy value it a lot more attractive. A major key factor in my family was earning to a BTC. The novelty of having a real gold coin was mined out of one dollar amount. This is a major change when people who rely for money on fake economic markets, it very often means fraud can get took out of the system. People buy and take virtual tokens the way they accept the money or not. They will still most likely not pay you. Now I know all about that. Dattaby, it was called, is billed as the biggest event in cycling's history. The race, known as Dattaby, will launch on 13, resulting in more than 25{,}000 journeys. The online stage chain Dattaby Tour Organizers has selected SportStadia, DeutschesStadia, and Hy\"{a}ntra for over funding paid for taking part in the event. The special event will be expanded at the second phase and planned in a third. Next year a consortium of major airports will run on London between Channel 1 and Biao Town USA, which will open three days in two phases for spring months later. ``At the end of the day, cyclists ride an atmosphere of than 2{,}000 miles of wild bikes and it will bring unprecedented excitement to Europe,'' said CEO Gerry O'Sullivan of CyclingCubb. ``A national tour of Cycling is always on us. This is to be the first event in the world to take the road from London to coast back on Super France 23 for the June 21 Addero Weekend.'' Dattawa organisers added: ``We are amazed by what all the elite cyclists will be on the three-race journey to see as the race moves towards Europe, and nothing beats a day travelling out to a bloke or the morning's coffee.'' ``Pro cyclists from across the world will have the opportunity to take in their cycling tours with non progenitor members all over the world, and the food takeaway will become part of the thrilling food. Exhilarating accommodations along with the backbone of our cycling infrastructure, this Dattaby will put all the importance of our highly motored cycling efforts into the event. We have invited some fantastic teams together throughout our careers to compete throughout the route internationally that we are invited to compete in. We are delighted to see all cyclists available today from June, concentrating in part on keeping an authentic approach to the sport.'' Throughout the season there will be seven-star hotels, while hotel rooms will be used by a nighttime range of luxury, racing and fitness options. Other events happen at times in London on the same day during the season. Share this article Print Facebook Twitter Email Share this article to your email address. Show Email Click a link to the Sutton Trusty Watson Watson is demanding full access with their Travel Drivers Service. Officers have targeted people in the wake of death of a man suspected of allegedly trying to connect Uber with Uber. Dozens of arrests during Beesworth Road and Rigswell Road in Waterloo area were received today. Their own company is now not providing an investigation into their activities with the public but will be providing a drive all the way through to prepare a ticket to the event. said CEO The team is proud of our classic. By helping to organise our event, we will no doubt celebrate an elite cycling coach working in America. It

\end{tcolorbox}

\clearpage 


\begin{tcolorbox}[
  enhanced,
  colback=black!3,
  colframe=black!55,
  boxrule=0.6pt,
  arc=2pt,
  boxsep=8pt,
  left=12pt, right=12pt, top=8pt, bottom=12pt,
  toptitle=10pt,
  bottomtitle=8pt,
  title={%
    \textbf{DiLaDiff}\quad
    $\Nc=5$,\quad $\Nd=64$\quad
    $\cdot$\quad
    \textbf{GenPPL:}~97.2\quad
    $\cdot$\quad
    \textbf{Entropy:}~5.50
  },
  fonttitle=\small\sffamily,
  coltitle=black,
  colbacktitle=black!10,
  titlerule=0.5pt,
]
\small\setlength{\parskip}{0.45em}\setlength{\parindent}{0pt}
\raggedright

nature of the comment, a speech made by Jay Leno when he split in their favour. Alan agreed in his own, of course, of winning his votes (the chart chart shows above). Nevertheless, he was able to achieve just argumentation, with no general exertation or action from class. Nonetheless, Apeball was clear that he was simply suggesting rather than implementing, to say, Wall StreetNet is attempting to force all all its members to negotiate the TLDR \$1.3 trillion tax accord to GNU senators at levels far exceed the goals the party hopes voters to reach. Furthermore, it is logical that the elaborate agreement would enable the Republican to pursue a second deal through Common Core while including millions of people paying taxes, placing thousands of dollars on allegedly fraudulent schemes on their tax rolls and even as unlikely candidates for the House. At approximately the same time, the commentspunge upon the decision of the party's four chapters, who according to Wall StreetNet, declared that he disassociated, shortly after, he had lied from a pre-peace and talks about a tax deal, ``looked at as a step into of various Heavens.'' Quite reasonable gesture for the group, although it is clear that the organization is now threatening to exit the tea party's contract with his allies, he has had no right to give up of the accounts during the 2014 elections. ``All the members are already publicly acknowledging the alleged `overspending' and beginning to pretend it is an accounting,'' Folder adds. The party is actually now threatening to go ahead with the plan by using spin a nebbusbitividiensis in principle as the payer in Roberto Aamer's finances. Simply put, a group of wealthy interests would acquire 5.99 millionaire shares in bonds to cover either a bill or raise its annual deficit. At that point, the on their affairs would continue from President Trumanes to the Cayman Islands, thus reducing the risk. Apparently they would have to go away if Aamer believes he is in the group with his latest proposal. The details have been denied continued negotiations. Ron Paul A.\ Hycock: It has been reported that the Fitz Caucus group is in contact with a giant group representing some of the voices demanding transparency in the budget: budget cuts, a border wall, and a U.S.\ dollar, or part of which to House President China would cost about \$1.3 trillion. Media Resources: This is the view of the House Judiciary Committee on A Comprehensive Firm to De-Define the USary Sessions Report: This featured article is written by Georgia Liberty Civil Society Supreme Court Judge George C.\ Francis O.\ Lindsay. The cannabis coalition is lauded by online news outlets, covform, and local events, and constantly feeds by Christianity's support of the elections. Naturally follows the conservative Radio host's outrage over the ``absurd'' decision to stop bandanna Taylor Swift taking on ``booing and rude shouting,'' the Thread in Limbaugh's host and podcast hosts, Bobby Pugh, canceled his radio show. Host Carney fired back from the offensive by complaining that the remark ``offensive,'' a cringely demeans former rock singer Ceremonian Swift. Jordan Levivin, co-chairman of Jimmy Limmel, said Pelley had already had informal conversations with Pelley, two days after the remark, which began angering him on social media: ``The goal is of a marriage equality agenda, for the party, to let people know what he she believed. This is essentially a gargantuan political push push for the failed grassroots Democratic Party on marriage---that is literally the end of one's life, rather than a legitimate overgovernment. According to Opinion and other commenters, Plante appears on his Facebook to show himself a `champion of civility' and according to a panel of leaders, `there is sensitivity on the price of food.''' Melania Annouched against the remark, saying, ``They're not talking about political unity, they're not going to say they're talking about at all.'' Ironically, two prominent members of Limbaugh Radio's network, Beck, stood to--face over the comment when he told Hannity that she shared the tweet publicly. ``It was disgusting to hear Carney say anything like that in front of me,'' Beck said. ``Clearly not.'' That apparently provoked the singer to engage his own theory of feminism, with Swift's remarks he wanted a woman to have a life, after Hannity pushed back on the offensive night. Moments after the conversation took place, pundits on the left quickly pointed to a few ``conservative'' chat rooms who allegedly poured gasoline on Carney's air on social, and blasted Hannity's claim they're liberal, because he supposedly opposed Swift.

\end{tcolorbox}

\end{document}